\documentclass[nohyperref]{article}

\usepackage{amsmath,amsfonts,bm}

\def\eqref#1{equation~\ref{#1}}

\def\1{\bm{1}}

\DeclareMathAlphabet{\mathsfit}{\encodingdefault}{\sfdefault}{m}{sl}
\SetMathAlphabet{\mathsfit}{bold}{\encodingdefault}{\sfdefault}{bx}{n}

\def\sR{{\mathbb{R}}}

\newcommand{\E}{\mathbb{E}}

\newcommand{\R}{\mathbb{R}}

\usepackage{microtype}
\usepackage{graphicx}
\usepackage{subfigure}
\usepackage{dsfont}
\usepackage{array}
\usepackage{booktabs} %

\usepackage{hyperref}

\usepackage[accepted]{icml2022}

\usepackage{amsmath}
\usepackage{amssymb}
\usepackage{mathtools}
\usepackage{amsthm}

\usepackage[capitalize,noabbrev]{cleveref}

\theoremstyle{plain}
\newtheorem{theorem}{Theorem}[section]

\newtheorem{lemma}[theorem]{Lemma}

\theoremstyle{definition}

\newtheorem{assumption}[theorem]{Assumption}
\theoremstyle{remark}

\newcommand{\xx}{\ensuremath\mathbf{x}}

\newcommand{\norm}[1]{\ensuremath{\left\| #1 \right \|}}

\usepackage{amssymb}
\usepackage{enumitem}

\newcommand{\qcr}[1]{{\fontfamily{qcr}\selectfont #1}}
\newcommand{\myparagraph}[1]{
\vspace{0.1cm}\noindent
\textbf{#1.}
}

\usepackage{pifont}
\newcommand{\cmark}{\ding{51}}%
\newcommand{\xmark}{\ding{55}}%

\usepackage[textsize=tiny]{todonotes}

\icmltitlerunning{ProgFed: Effective, Communication, and Computation Efficient Federated Learning
by Progressive Training}

\begin{document}

\twocolumn[
\icmltitle{ProgFed: Effective, Communication, and \\ Computation Efficient Federated Learning
by Progressive Training}

\icmlsetsymbol{equal}{*}

\begin{icmlauthorlist}
\icmlauthor{Hui-Po Wang}{cispa}
\icmlauthor{Sebastian U. Stich}{cispa}
\icmlauthor{Yang He}{cispa}
\icmlauthor{Mario Fritz}{cispa}
\end{icmlauthorlist}

\icmlaffiliation{cispa}{CISPA Helmholz Center for Information Security, Germany}

\icmlcorrespondingauthor{Hui-Po Wang}{hui.wang@cispa.de}

\icmlkeywords{Machine Learning, ICML}

\vskip 0.3in
]

\printAffiliationsAndNotice{}  %

\newcommand{\mario}[1]{{\color{green}#1}}
\newcommand{\yang}[1]{{\color{cyan}#1}}

\newcommand{\huipo}[1]{{#1}}
\newcommand{\hpw}[1]{\todo[color=red!20,size=\footnotesize,caption={}]{huipo: #1}{}}
\newcolumntype{P}[1]{>{\centering\arraybackslash}p{#1}}

\begin{abstract}
Federated learning is a powerful distributed learning scheme that allows numerous edge devices to collaboratively train a model without sharing their data.
However, training is resource-intensive for edge devices, and limited network bandwidth is often the main bottleneck.
Prior work often overcomes the constraints by condensing the models or messages into compact formats, e.g., by gradient compression or distillation.
In contrast, we propose ProgFed, the first progressive training framework for efficient and effective federated learning. It inherently reduces computation and two-way communication costs while maintaining the strong performance of the final models. We theoretically prove that ProgFed converges at the same asymptotic rate as standard training on full models. Extensive results on a broad range of architectures, including CNNs (VGG, ResNet, ConvNets) and U-nets, and diverse tasks from simple classification to medical image segmentation show that our highly effective training approach saves up to $20\%$ computation and up to $63\%$ communication costs for converged models. As our approach is also complimentary to prior work on compression, we can achieve a wide range of trade-offs by combining these techniques, showing reduced communication of up to $50\times$ at only $0.1\%$ loss in utility. Code is available at \url{https://github.com/hui-po-wang/ProgFed}.
\end{abstract}

\section{Introduction}
\label{sec:intro}
Federated Learning (FL) has led to remarkable advances in the development of extremely large machine learning systems~\citep{mcmahan2017communication}. Federated training methods allow multiple clients (edge devices) to jointly train a global model without sharing their private data with others. 
Training methods in FL suffer from high communication and computational costs, as edge devices are often equipped with limited hardware resources and limited network bandwidth.

Prior literature has studied various compression techniques to address the computation and communication bottlenecks.
These methods can be divided into three main categories (see also Table~\ref{table:comp_all_compression}):
(i)~Compression techniques that represent gradients (or parameters) with fewer bits to reduce communication costs. Prominent examples are quantization~\citep{alistarh2017qsgd, lin2018deep, fu2020don} or sparsification~\citep{Stich2018:sparsifiedSGD, konevcny2016federated}.
(ii)~Model pruning techniques that identify (much smaller) sub-networks within the original models to reduce computational cost at inference~\citep{li2019fedmd, lin2020ensemble}.  
And (iii)  knowledge distillation~\citep{hinton2015distilling} techniques that allow the server to distill the knowledge from the clients with hold-out datasets~\citep{li2019fedmd, lin2020ensemble}.
Despite the significant progress, most of them rarely leverage learning dynamics to reduce resource demands.

Prior work has observed that neural networks tend to stabilize from shallower to deeper layers during training~\citep{raghu2017svcca}. This behavior provides additional scope to improve resource demands. We take advantage of this opportunity by leveraging progressive learning~\citep{karras2018progressive}, a well-known technique in image generation. It first trains the shallower layers on simpler tasks (e.g., images with lower resolution) and gradually grows the network to tackle more complicated tasks (e.g., images with higher resolution). The growing process inherently reduces computation and communication costs when the models are shallower. Despite the appealing features, current success mainly focuses on centralized training, and no previous study has systematically investigated exploiting progressive learning to reduce the costs in federated learning.\looseness=-1

\begin{table*}[btp]
\caption{Comparison of ProgFed to other compression schemes.}
\label{table:comp_all_compression}
\vskip 0.15in
\centering
{\small 
\begin{tabular}{@{}lP{3cm}P{3cm}P{3cm}@{}}
\toprule
Technique
 & \begin{tabular}[c]{@{}c@{}}Computation \\ Reduction\end{tabular} & \begin{tabular}[c]{@{}c@{}}Communication \\ Reduction\end{tabular} & \begin{tabular}[c]{@{}c@{}}Dataset \\ Efficiency\end{tabular} \\ \midrule
Message Compression & \xmark & \cmark & \cmark \\
Model Pruning & \cmark (only for inference) & \xmark & \cmark \\
Model Distillation & \cmark & \cmark & \xmark \\
ProgFed (Ours) & \cmark & \cmark & \cmark \\ \bottomrule
\end{tabular}
}
\vskip -0.1in
\end{table*}

We propose ProgFed, the first federated progressive learning framework that reduces both communication and computation costs while preserving model utility.
Our approach divides the model into several overlapping partitions and introduces lightweight local supervision heads to guide the training of the sub-models. 
The model capacity is gradually increased during training until it reaches the full model of interest. Due to the nature of progressive learning, our method can reduce computational overheads and provide two-way communication savings (both from server-to-client and client-to-server directions) since the shallow sub-models have much fewer parameters than the complete model.

We show that ProgFed converges at the same asymptotic rate as standard training on full models.
Extensive results show that our method can resemble and sometimes outperform the baselines using much fewer costs. Moreover, ProgFed is compatible with classical compression, including sparsification and quantization, and various federated optimizations, such as FedAvg, FedProx, and FedAdam. These results confirm the generalizability of ProgFed and could motivate more advanced FL compression schemes based on progressive learning.

We summarize our main contributions as follows.
\begin{itemize}[leftmargin=12pt,nosep]
    \item We propose ProgFed, the first federated progressive learning framework to reduce the training resource demands (computation and two-way communication). We show that ProgFed converges at the same asymptotic rate as standard training the full model.
    \item We conduct extensive experiments on various datasets (CIFAR-10/100, EMNIST and BraTS) and architectures (VGG, ResNet, ConvNets, 3D-Unet) to show that with the same number of epochs, ProgFed saves around $25\%$ computation cost, up to $32\%$ two-way communication costs in federated classification, and $63\%$ in federated segmentation without sacrificing performance.
    \item Our method allows to reduce communication costs around $2\times$ in classification and $6.5\times$ in U-net segmentation while achieving practicable performance ($\geq 98\%$ of the best baseline). This is beneficial for combating limited training budgets in federated learning.
    \item We show that ProgFed is compatible with existing techniques. It complements classical compression to reduce up to $50\times$ communication costs at only $0.1\%$ loss in utility and can generalize to advanced optimizations with up to $4\%$ improvement over standard training.
\end{itemize}

\begin{figure*}[th!]
\centering
    \begin{tabular}{cc}
    \includegraphics[trim=14cm 8.4cm 14cm 7cm,clip,width=0.4\linewidth]{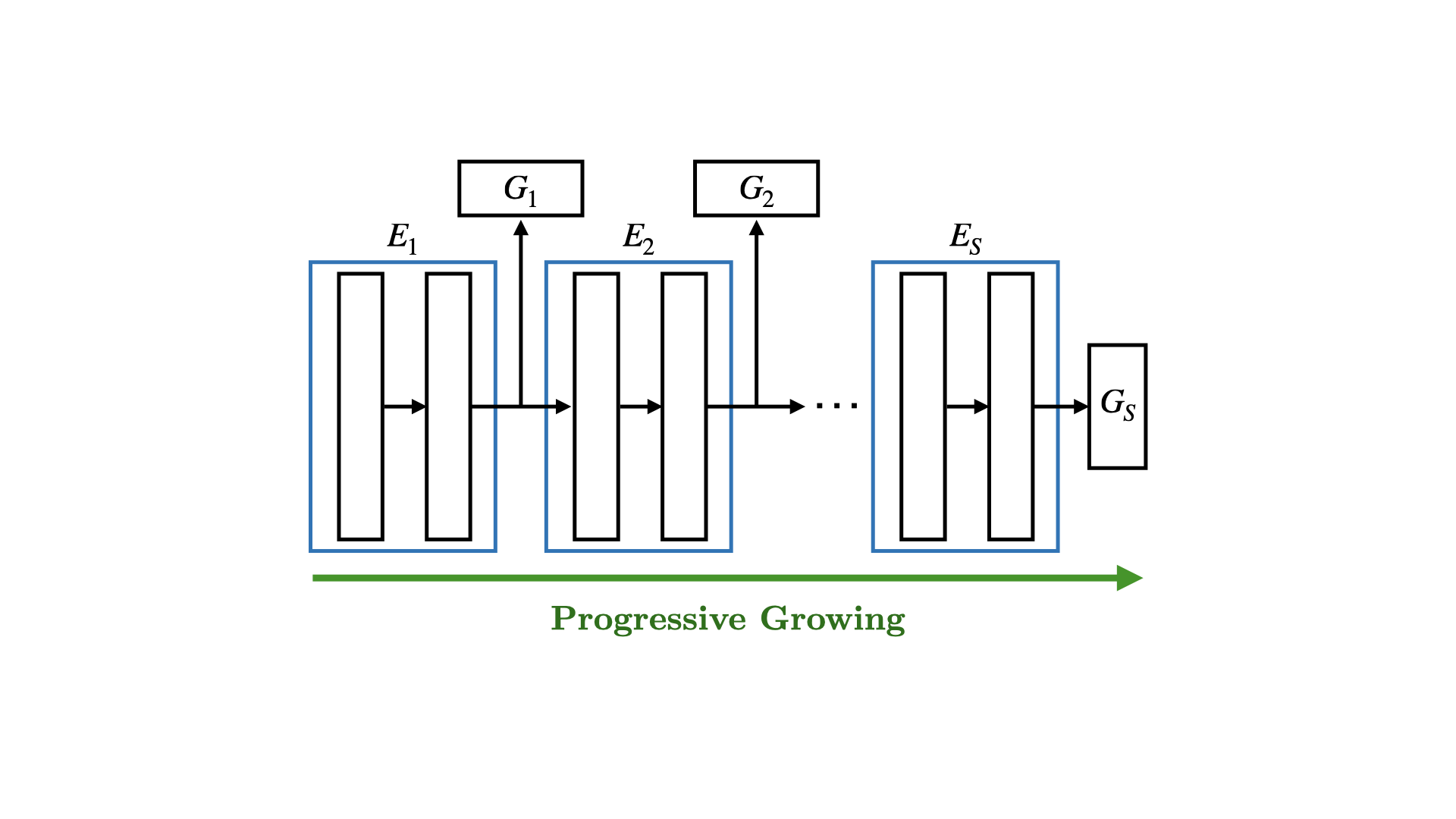} & \hspace{0.67cm} \includegraphics[trim=14cm 8.4cm 14cm 7cm,clip,width=0.4\linewidth]{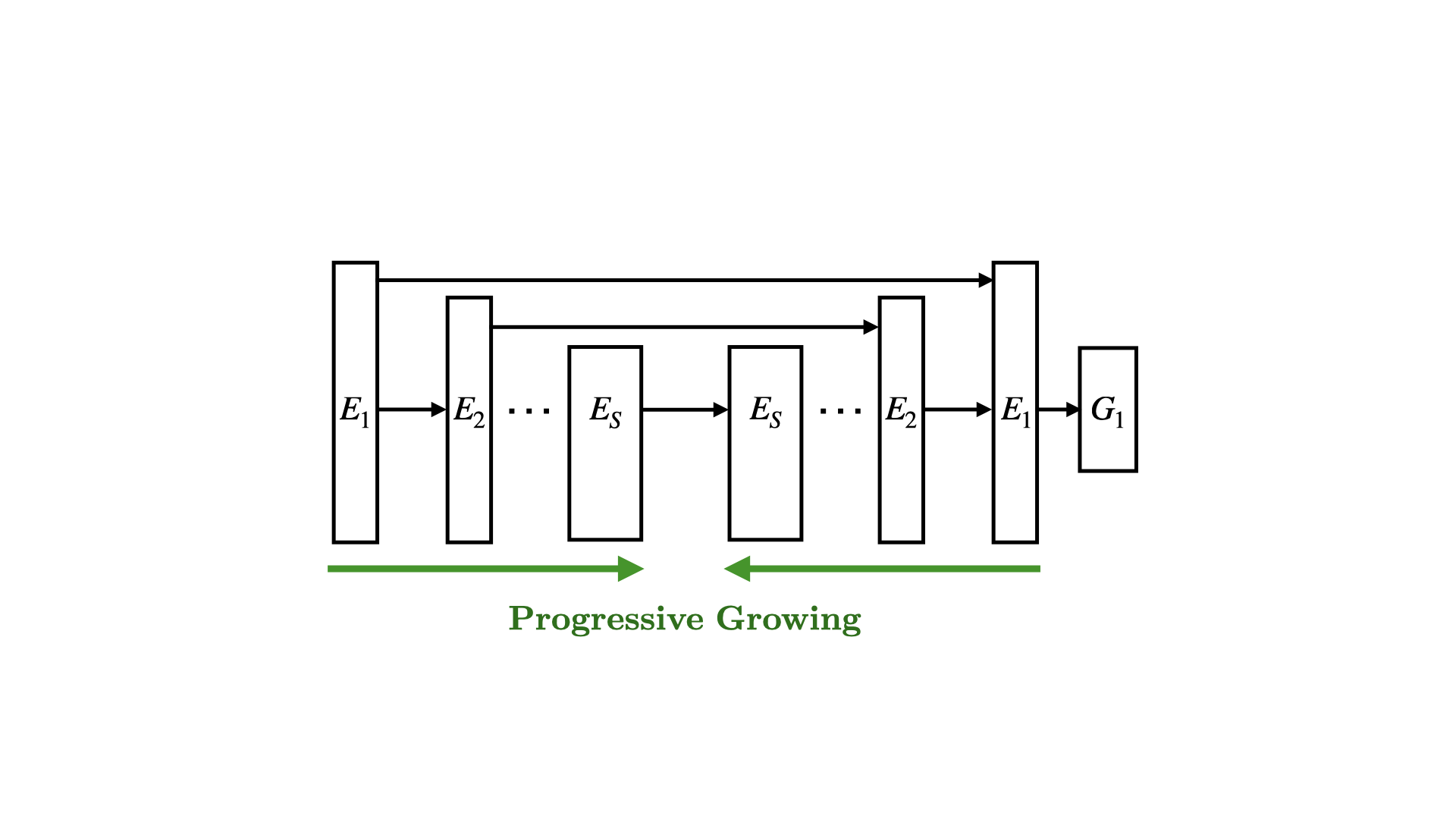} \\
{\footnotesize    (a) Feed-forward networks} &{ \footnotesize (b) U-nets (symmetric growing).}
\end{tabular}
\caption{An overview of ProgFed on (a) feed-forward networks and (b) U-nets (symmetric growing illustrated). We progressively train a deep neural network from the shallower sub-models, e.g. $\mathcal{M}^1$ consisting of the main block $E_1$ and head $G_1$ (Eq.~\ref{eq:def_submodel}), gradually expanding to the full model $\mathcal{M}^S=\mathcal{M}$ (Eq.~\ref{eq:def_model}). Note that the local heads $G_i$ in feed-forward networks are only used for training sub-models and discarded when progressing to the next stage.}
\label{fig:arch}
\vskip -0.1in
\end{figure*}

\section{Related Work}
\label{sec:related_work}

\myparagraph{Progressive Learning}
Progressive learning was initially proposed to stabilize training processes and has been widely considered in vision tasks such as image synthesis~\citep{karras2018progressive}, image super-resolution~\citep{wang2018fully}, facial attribute editing~\citep{wu2020cascade} and representation learning~\citep{li2019progressive}. The core idea is to train the model from easier tasks (e.g., low-resolution outputs or shallower models) to difficult but desired tasks (e.g., high-resolution outputs or deeper models). 
The benefit of progressive learning in cost reduction, namely lower resource demands when the models are shallow, has not been explored in FL. Recent work has investigated partial updates for FL, but they either fall in greedy layer-wise updates~\citep{belilovsky2020decoupled} or are specific to certain models, e.g., transformers~\citep{he2021pipetransformer}. In this work, we systematically investigate the application of progressive learning for general federated tasks and demonstrate its convergence rate and compatibility.

\myparagraph{Message compression} Prior work has studied message compression (e.g., on gradients or model weights) to reduce the communication costs in distributed learning. The first category focuses solely on client-to-server compression~\citep{alistarh2017qsgd, wen2017terngrad, lin2018deep,  bernstein2018signsgd, Stich2018:sparsifiedSGD,  konevcny2016federated, karimireddy2019error, fu2020don, Stich2020:compression}. To name a few, \citet{konevcny2016federated} reduce the costs by sending sparsified gradients and compressing them with probabilistic quantization. \citet{alistarh2017qsgd, wen2017terngrad} prove the convergence of probabilistic quantized SGD. SignSGD~\citep{bernstein2018signsgd, karimireddy2019error} significantly compresses the gradients with only one bit.
Compression of server-to-client communication is non-trivial and has been a focus of recent research~\citep{yu2019double_quantization,tang2019doublesqueeze,liu2020double_residual,philippenko2020bidirectional,he2021cossgd}. Instead of dedicated designs, our method inherently reduces two-way communication costs and complements existing methods, as we will show in Section~\ref{sec:exp}.

\myparagraph{Model pruning and distillation} In addition to compressing messages, model pruning discards redundant weights according to diverse criteria~\citep{mozer1989skeletonization,lecun1990optimal,frankle2018the,lin2019dynamic}, while model distillation~\citep{bucilua2006model,hinton2015distilling} transmits logits rather than gradients to improve communication efficiency~\citep{li2019fedmd, lin2020ensemble, he2020group, choquette2020capc}. However, the former usually happens after training, and the latter one either requires additional query datasets~\citep{li2019fedmd, lin2020ensemble} or cannot enjoy the merit of datasets from different sources~\citep{choquette2020capc}. In contrast, ProgFed stays dataset-efficient and reduces resource demands throughout training.

\myparagraph{Early-exit networks} Early-exit networks~\citep{kaya2019shallow,scardapane2020should,teerapittayanon2016branchynet} are equipped with multiple output branches. Each data sample can opt for different branches at test time, thus reducing the computation costs of inference. Despite the efficiency at test time, early-exit networks~\citep{scardapane2020should,teerapittayanon2016branchynet} often consume more computation power during training since they have to maintain all auxiliary classifiers (heads). On the other hand, our method is designed for computation and communication cost reduction at \emph{training} time. We discard temporal heads and consume fewer costs than the entire and early-exit networks. These features outline the main difference between our method and early-exit networks.

\section{ProgFed}
\label{sec:method}
\huipo{In this work, we leverage the learning dynamics that networks stabilize from shallower to deeper layers for cost reduction. Motivated by progressive learning, we propose ProgFed that progressively expands the network from a shallower one to the complete model. We provide convergence analysis on single-client training for conciseness while it stays efficient and can generalize to federated multi-client training. The proposed model splitting and progressive growing are illustrated in Figure~\ref{fig:arch}, and the federated optimization scheme is summarized in Algorithm~\ref{alg:progfed}.
We now present the proposed method in detail below.}

\subsection{Proposed Method}
\label{ssec:proposed_method}

 \myparagraph{Model Structure} We now describe the proposed training method. For a given a machine learning model $\mathcal{M}$, i.e.\ a function $\mathcal{M}(\cdot,\xx) \colon \sR^n \rightarrow \sR^k$ that maps $n$-dimensional input to $k$ logits for parameters (weights) $\xx \in \sR^d$, we assume that the network can be written as a composition of blocks (feature extractors) $E_i$ along with a task head $G_S$, namely,
\begin{equation}
\label{eq:def_model}
    \mathcal{M} \vcentcolon=  G_S \circ \mathop{\bigcirc}^S_{i=1} E_i  = G_S \circ E_S \circ \dotsb \circ E_2 \circ E_1 \,.
\end{equation}
Note that the $E_i$'s could denote e.g., a stack of residual blocks or simply a single layer. 
The learning task is solved by minimizing a loss function of interest $\mathcal{L}: \sR^k \rightarrow \sR$ (e.g., cross-entropy) that maps the predicted logits of $\mathcal{M}$ to a real value, i.e.\ minimization of $f(\xx):= \mathcal{L} \circ \mathcal{M}(\xx)$.

 \myparagraph{Progressive Model Splitting} To achieve progressive learning, we first divide the network $\mathcal{M}$ into $S$ stages, denoted by $\mathcal{M}^s$, for $s \in \{1, \dots, S\}$ associated with the split indices. We additionally introduce local supervision heads for providing supervision signals. Formally, we define
\begin{equation}
\label{eq:def_submodel}
    \mathcal{M}^s \vcentcolon= G_s \circ \mathop{\bigcirc}_{i=1}^s E_i \,,
\end{equation}
where $G_s$ is a newly introduced head. Each head $G_s$, for $s < S$ consists of only a pooling layer and a fully-connected layer in our experiments for feed-forward networks. The motivation is that simpler heads may encourage the feature extractors $E_i$ to learn more meaningful representations. Note that the sub-model $\mathcal{M}^s:\sR^n \rightarrow \sR^k$ produces the same output size as the desired model $\mathcal{M}$; therefore, its corresponding loss  $f^s(\xx^s):= \mathcal{L}\circ \mathcal{M}^s(\xx^s)$ can be trained with the same loss criterion $\mathcal{L}$ as the full model. \looseness=-1
 
 \myparagraph{Training of Progressive Models} 
We propose to train each sub-model $\mathcal{M}^s$ for $T_s$ iterations (a certain fraction of the total training budget) and gradually grow the network from $\mathcal{M}^1$ to $\mathcal{M}^S=\mathcal{M}$.
When growing the network from stage $s$ to $s+1$, we pass the corresponding parameters of the pretrained  blocks $E_i$, $i\leq s$, to the next model $\mathcal{M}^{s+1}$ and initialize its blocks $E_{s+1}$ and $G_{s+1}$ with random weights.
Once the progressive training phase is completed, we continue training the full model $\mathcal{M}$ in an end-to-end manner for the remaining iterations.
The length $T_s$ of each progressive training phase is a parameter that could be fine-tuned individually for each stage (depending on the application) for best performance. 
However, as a practical guideline that we adopted for all experiments in this paper, we found that denoting roughly half of the total number of training iterations $T$ to progressive training, and setting $T_s = \smash{\frac{T}{2S}}$ for $s < S$, $T_S = \smash{\frac{T(S+1)}{2S}}$, such that $T=\smash{\sum_{s=1}^S} T_s$, works well across all considered training tasks. 

 \myparagraph{Extension to U-nets} In addition to feed-forward networks (Figure~\ref{fig:arch}(a)), we show that our method can generalize to U-net (Figure~\ref{fig:arch}(b)). U-net typically consists of an encoder and a decoder. Unlike feed-forward networks, the encoder sends both the final and intermediate features to the decoder. Therefore, we propose to grow the network from outer to inner layers as shown in Figure~\ref{fig:arch}(b) and retain the original output heads as $G_i$. We refer to the strategy as the \textit{Symmetric} strategy. In contrast, we propose another baseline, the \textit{Asymmetric} strategy, which adopts the full encoder at the beginning and gradually grows until it reaches the full decoder. For this strategy, we also adopt several temporal heads for earlier training stages. As we will show in Section~\ref{ssec:communication_efficiency}, the \textit{Symmetric} strategy significantly outperforms the \textit{Asymmetric} strategy, which supports the notion of progressive learning. \looseness=-1

 \myparagraph{Practical considerations} We empirically observe that learning rate restart~\citep{loshchilov2016sgdr} facilitates training in the centralized setting. This is because sub-models may overfit the local supervision while learning rate restart helps the sub-models escape from the local minima and quickly converge to a better minima. On the other hand, warm-up~\citep{goyal2017accurate} for the new layers plays an important role in federated learning. Model weights often take a longer time to converge in federated learning, which makes the newly added layers introduce more noise to the previous layers. With warm-up, the new layers recover the performance without affecting previous layers. In particular, warm-up leads to around 2\% difference (53.23 vs. 51.09) on CIFAR-100 with ResNet-18.

\begin{figure}[t]
\vskip -0.1in
\begin{algorithm}[H]
\caption{ProgFed---Progressive training in a Federated Learning setting}\label{alg:progfed}
\begin{algorithmic}[1]
\STATE{\textbf{Input:} initialization $\xx_0^1$, model $\mathcal{M}(\cdot, \xx_0)$, iteration budgets $T$, $T_s$, number of stages $S$, $s = 1$, desired number of  local updates $J \geq 1$, learning rate $\eta$}
\STATE{\textbf{Output:} parameters $\mathbf{x}_T$ and trained model $\mathcal{M}(\cdot, \xx_T)$}
\FOR{$t = 1,\dots, T$}
    \STATE // Switch from $\mathcal{M}^{s}$ to  $\mathcal{M}^{s+1}$ after $T_s$ iterations
    \IF{$\min(S, \lceil \frac{t}{T_s} \rceil) > s$}
        \STATE // Initialize new block $E_{s+1}$ and new head $G_{s+1}$
        \STATE initialize parameter $\xx_t^{s+1}$ randomly
        \STATE // Copy parameters of shared blocks $E_1, \dots, E_s$.
        \STATE $\xx^{s+1}_{t \mid E_{s}}\gets \xx^s_{t \mid E_s}$
        \STATE $s \gets s + 1$ \hfill (the old head $G_s$ is discarded)
    \ENDIF
    \STATE // Standard Federated Learning on active model $\mathcal{M}^s$
    \STATE Sample a subset $\mathcal{C}$ of clients 
    
    \FOR{each active client $c \in \mathcal{C}$}
        \STATE initialize $\mathbf{x}^s_{c, 1} \gets \mathbf{x}_t^s$ \hfill (send $\xx_t^s$ to active clients)
        \IF{warm-up is needed after growing}
            \STATE // Warm up the newly added layers
            \STATE freeze $\xx^s_{c,1 \mid E_{s-1}}$ and warm-up $\xx^s_{c,1 \mid E_{s-1}^\complement}$ 
        \ENDIF
        \FOR{$j = 1, \dots, J$ \hfill $\triangledown$ Local SGD updates}
            \STATE // Compute (mini-batch) gradient $g_c^s$ on client $c$'s data
            \STATE $\mathbf{x}^s_{c, j+1} = \mathbf{x}^s_{c, j} - \eta g^s_{c}(\xx^s_{c, j})$ 
        \ENDFOR
        \STATE $\Delta_c = \mathbf{x}_{c, J} - \mathbf{x}_{c, 1}$
    \ENDFOR
    \STATE // Aggregate updates from the clients
    \STATE $\mathbf{x}^s_{t+1} = \mathbf{x}^s_{t} + \frac{1}{|\mathcal{C}|}\sum_{c=1}^{\mathcal{C}} \Delta_c$ 
\ENDFOR
\end{algorithmic}
\end{algorithm}
\vskip -0.2in
\end{figure}
 
\subsection{Convergence Analysis}
\label{ssec:thm}
In this section we prove that progressing training converges to a stationary point at the rate $\mathcal{O}\bigl(\frac{1}{\epsilon^2}\bigr)$, i.e.\ with the same  asymptotic rate as SGD training on the full network. For this, we extend the analysis from  \citep{mohtashami2021simultaneous} that analyzed the training of partial subnetworks. However, in our case the networks are not subnetworks, $\mathcal{M}^{s} \not\subset \mathcal{M}^{s+1}$ (as the head is not shared), and we need to extend their analysis to progressive training with different heads.

\myparagraph{Notation} 
We denote by $\xx^s$ the parameters of $f^s$, $s \in \{1,\dots,S\}$ and abbreviate $\xx^S=\xx \in \R^d$ for convenience. For $s \leq i \leq S$ let
${x^i}_{\mid E_s}$ and $\nabla f^i(\xx^i)_{\mid E_s}$ denote the projection of the parameter $\xx^i$ and gradient $\nabla f^i (\xx^i)$ to the dimensions corresponding to the parameters of $E_1$ to $E_s$ (without parameter of $E_{s+1}$ to $E_i$ and without head $G_i$).
In iteration $t$, the progressive training procedure updates the parameters of model $f^s$, $s = \min(S, \lceil \frac{t}{T_s}\rceil)$. For convenience, we do not explicitly write the dependency of $s$ on $t$ below, and use the shorthand $x_t^s$ to denote the corresponding model at timestep $t$. We further define $\xx_t$ such that $\xx_{t\mid E_s} = x^t_{s\mid E_s}$ and $\xx_{t \mid E_s^\complement} = \xx_{0 \mid  E_s^\complement}$ on the complement. %

\begin{assumption}[$L$-smoothness]
\label{asmp_lsmooth}
The function $f \colon \mathbb{R}^d \to \mathbb{R}$ is differentiable and there exists a constant $L>0$ such that
\begin{equation}
    \norm{\nabla f(\mathbf{x}) - \nabla f(\mathbf{y})} \leq L\norm{\mathbf{x}-\mathbf{y}}.
\end{equation}
\end{assumption}

We assume that for every input $\mathbf{x}^s$, we can query an unbiased stochastic gradient $g^s(\mathbf{x}^s)$ with $\E [g^s(\mathbf{x}^s)] = \nabla f^s(\mathbf{x}^s)$. We assume that the stochastic noise is bounded. %
\begin{assumption}[Bounded noise]
\label{asmp_msbound}
There exist a parameter $\sigma^2\geq0$ such that for any $s \in \{1, \dots, S\}$:

\begin{equation}
    \mathds{E}\norm{g^s(\mathbf{x}^s) - \nabla f^s(\mathbf{x}^s)}^2 \leq 
    \sigma^2\,, \qquad \forall \mathbf{x^s} \,.
\end{equation}
\end{assumption}

The progressive training updates $\xx_{t+1}^s = \xx^t_s - \gamma_t g^s(\xx_t^s)$ with a SGD update on the model $\xx^s_t$. With the two  assumptions above, which are standard in the literature, we prove the convergence of sub-models $\mathcal{M}^s$ as well as the model of interest $\mathcal{M}$.
\begin{theorem}
\label{thm:convergence}
Let Assumptions~\ref{asmp_lsmooth} and \ref{asmp_msbound} hold, and let the stepsize in iteration $t$ be 
 $\gamma_t=\alpha_t\gamma$ with $\gamma=\min\Big\{\frac{1}{L}, (\frac{F_0}{\sigma^2T})^\frac{1}{2}\Big\}$, $\alpha_t=$$\min\Big\{1, \frac{\langle \nabla f(\mathbf{x}_t)_{\mid E_s}, \nabla f^s(\mathbf{x}_t^s)_{\mid E_s}\rangle}{\| \nabla f^s(\mathbf{x}_t^s)_{\mid E_s} \|^2} \Big \}$. %
 Then it holds for any $\epsilon>0$,
\begin{itemize}[leftmargin=12pt,nosep]
    \item $\frac{1}{T}\sum^{T-1}_{t=0}\alpha_t^2\norm{\nabla f^s(\mathbf{x}_t^s)_{\mid E_s}}^2 < \epsilon$, after at most the following number of iterations T:
    \begin{equation}
        \mathcal{O}\left(\frac{\sigma^2}{\epsilon^2}+\frac{1}{\epsilon} \right) \cdot LF_0\,.
    \end{equation}
    \item 
    Let $q \vcentcolon= \max_{t\in[T]} \Big( q_t \vcentcolon= \frac{\| \nabla f(\mathbf{x}_t) \| }{\alpha_t \| \nabla f^s(\mathbf{x}_t^s)_{\mid E_s} \| } \Big)$, then $\frac{1}{T}\sum^{T-1}_{t=0}\norm{\nabla f(\mathbf{x}_t)}^2 < \epsilon$ after at most the following iterations $T$:
    \begin{equation}
        \mathcal{O} \left(\frac{q^ 4\sigma^2}{\epsilon^2}+\frac{q^2}{\epsilon} \right) \cdot LF_0\,,
    \end{equation}
    \end{itemize}
    where $F_0 \vcentcolon=f(\mathbf{x}_0)-(\min_\xx f(\xx))$.
\end{theorem}
Theorem~\ref{thm:convergence} shows the convergence of the full model $\mathcal{M}$. The convergence is controlled by two factors, the alignment factor $\alpha_t$ and the norm discrepancy $q_t$. The former term measures the similarity between the corresponding parts of the gradients computed from the sub-models and the full model (note that $\alpha_t \equiv 1$ in the last phase of training on $f^S=f$). The latter term $q$ measures the magnitude discrepancy (in Figure~\ref{fig:gradient_ratio} we display the evolution of $q_t$ during training for one example task, note that $q_t=1$ in the last phase of training).
We would like to highlight that the convergence criterion in the first statement is lower bounded by the average gradient in the last phase of the training, $\frac{1}{2} \cdot \frac{1}{2T}\sum^{T-1}_{t=T/2}\norm{\nabla f(\mathbf{x}_t)}^2 \leq \frac{1}{T}\sum^{T-1}_{t=0}\alpha_t^2 \| \nabla f^s(\mathbf{x}_t^s)_{\mid E_s} \|^2$ (this is due to our choice of the length of the phases, with $T_S \geq T/2$). This means, that progressive training will provably require at most twice as many iterations but can reach the performance of SGD training on the full model with much cheaper per-iteration costs.

\section{Experiments}
\label{sec:exp}
\subsection{Setup}
\label{ssec:setup}
We describe the main implementation details in this section and provide supplementary details in Section~\ref{sec:append-implementation}.

\myparagraph{Datasets, tasks, and models} We consider four datasets:\ CIFAR-10~\citep{krizhevsky2009learning}, CIFAR-100~\citep{krizhevsky2009learning}, EMNIST ~\citep{cohen2017emnist}, and BraTS~\citep{menze2014multimodal, bakas2017advancing, bakas2018identifying}. The former three are for image classification, while the last one is a medical image dataset for tumor segmentation. We conduct centralized experiments for analyzing the basic properties of our method while considering practical applications in federated settings. For the centralized settings, we train VGG-16, VGG-19~\citep{simonyan2014very}, ResNet-18, and ResNet-152~\citep{he2016deep} on CIFAR-100 (100 clients, IID). For the federated settings, we train ConvNets on CIFAR-10 and EMNIST (3400 clients, non-IID), ResNet-18 on CIFAR-100 (500 clients, non-IID), and 3D-Unet~\citep{sheller2020multi_brain_fl} on the BraTS dataset (10 clients, IID). Note that we follow \citet{hsieh2020non} to replace batch normalization in ResNet-18 with group normalization. 

\myparagraph{Implementation} We implement all settings with Pytorch~\citep{pytorch}. In the centralized experiments, we implement models based on \citet{devries2017improved}, where we run all experiments for 200 epochs and decay the learning rates in \textit{\{60, 120, 160\}} epochs by a factor of 0.1. We additionally adopt warm-up~\citep{goyal2017accurate} and learning rate restart~\citep{loshchilov2016sgdr} in our method to better fit in progressive learning. For federated classification, we follow federated learning benchmarks in \citep{mcmahan2017communication, reddi2021adaptive_fl} to implement CIFAR-10, CIFAR-100, and EMNIST, respectively. For federated tumor segmentation, we follow \citet{sheller2020multi_brain_fl} for the settings and data splits. We run 1500 epochs for EMNIST, 2000 epochs for CIFAR-10, 3000 epochs for CIFAR-100, and 100 epochs for BraTS. We set $S=3$ for EMNIST and $S=4$ for all the other datasets and $T_s$ as the practical guideline described in Section~\ref{sec:method}. We adopt 5 and 25 warm-up epochs for federated EMNIST and federated CIFAR-100, respectively. We sample clients without replacement in the same communication round but with replacement across different rounds~\citep{reddi2021adaptive_fl}. In each round, we sample 10 out of 100 clients for CIFAR-10 (IID), 40 out of 500 clients for CIFAR-100 (non-IID), 68 out of 3400 clients for EMNIST (non-IID), and 3 out of 10 clients for BraTS (IID). The global models are synchronously updated in all tasks.

\subsection{Computation Efficiency}
We first analyze the computation efficiency of our method in the centralized setting (where all data is available on a single device) to study the effect of the progressive training in isolation before moving to the federated use cases. We average the 
outcomes over three random seeds and consider four architectures on CIFAR-100, including VGG-16, VGG-19, ResNet-18, and ResNet-152.
As shown in Table~\ref{table:centralized_clf}, our method performs comparably to the baselines (that train on the full model) after 200 epochs while consuming fewer floating-point operations per second (FLOPs) and training wall-clock time. \looseness=-1

To analyze the efficiency, we report the performance when consuming different levels of costs. Figure~\ref{fig:acc-flops} shows that our method (orange lines) consistently lies above end-to-end training on the full model (blue lines), meaning that our method consumes fewer computation resources to improve the models. Moreover, we visualize $98\%$, $99\%$, $99.95\%$, and the best of the performance of the converged baseline (analysis with a larger range is presented in Figure~\ref{fig:append-cen-accer}). Figure~\ref{fig:cen-accer} indicates that our method improves computation efficiency across architectures. In the best case, our method can accelerate training up to 7$\times$ faster when considering limited computation budgets. We also observe that VGG models improve more than ResNets. A possible reason might be that due to local supervision, sub-models enjoy larger gradients compared to end-to-end training, while it rarely benefits ResNets since skip-connections could partially avoid the problem.  \looseness=-1

\begin{table}[t]
\caption{Results on CIFAR-100 in the centralized setting.}
\vskip 0.1in
\label{table:centralized_clf}
\centering
{\footnotesize
\begin{tabular}{@{}p{1.3cm}cccc@{}}
\toprule
           & \multicolumn{2}{c}{Accuracy} & \multicolumn{2}{c}{Reduction} \\ \midrule
           & End-to-end      & Ours       & Walltime      & FLOPs         \\ \cmidrule(l){2-5} 
ResNet18  & \textbf{76.08$\pm$0.12}           & 75.84$\pm$0.28      & -24.75\%      & -14.60\%      \\
ResNet152 & 77.77$\pm$0.38           & \textbf{78.57$\pm$0.33}      & -22.75\%      & -19.68\%      \\
VGG16 & \textbf{71.79$\pm$0.15}           & 71.54$\pm$0.45      & -14.57\%      & -13.02\%      \\
VGG19 & 70.81$\pm$1.18           & \textbf{70.90$\pm$0.43}      & -22.10\%      & -14.43\%      \\ \bottomrule
\end{tabular}
}
\vskip -0.1in
\end{table}

\begin{figure}[t]
    \centering
    \includegraphics[width=1\linewidth]{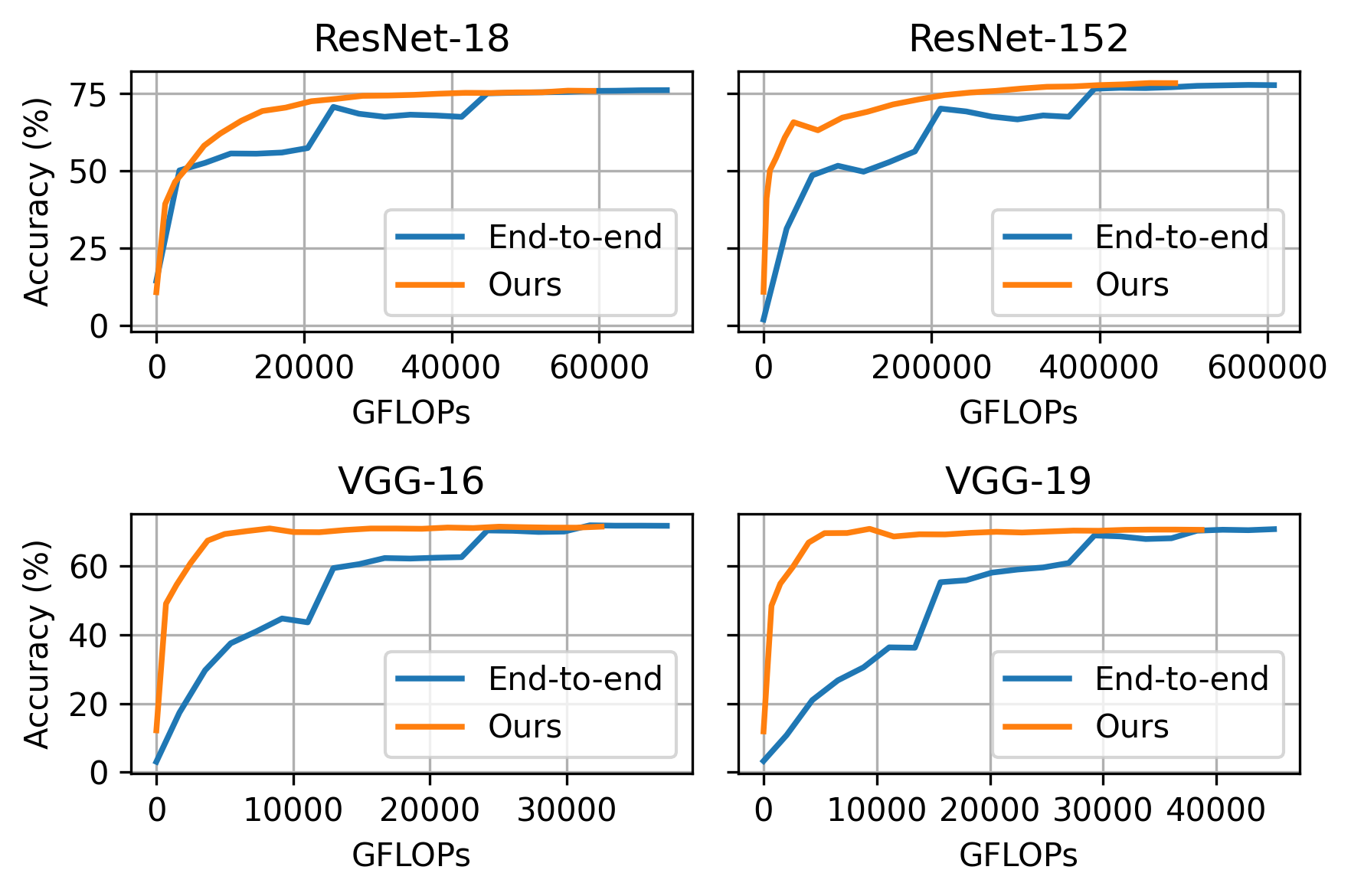}
    \vskip -0.12in
    \caption{Accuracy (\%) vs. GFLOPs on CIFAR-100 in the centralized setting.}
    \label{fig:acc-flops}
    \vskip -0.15in
\end{figure}

\begin{figure}[t]
    \centering
    \includegraphics[width=0.9\linewidth]{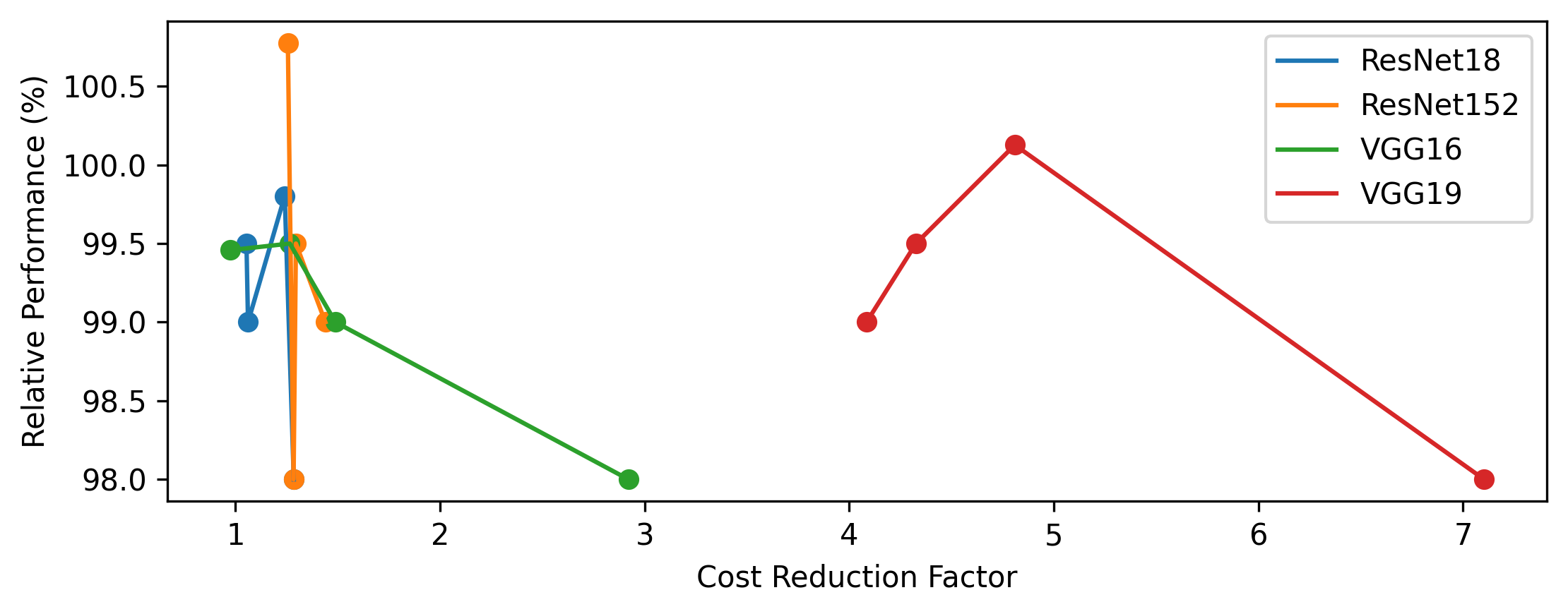}
    \vskip -0.12in
    \caption{Computation cost reduction at $98\%$, $99\%$, $99.95\%$, $\textit{best}$ compared to the baseline (training full models) performance in the centralized setting on CIFAR-100.} 
    \label{fig:cen-accer}
    \vskip -0.15in
\end{figure}
\begin{figure}[t]
    \centering
    \includegraphics[width=0.9\linewidth]{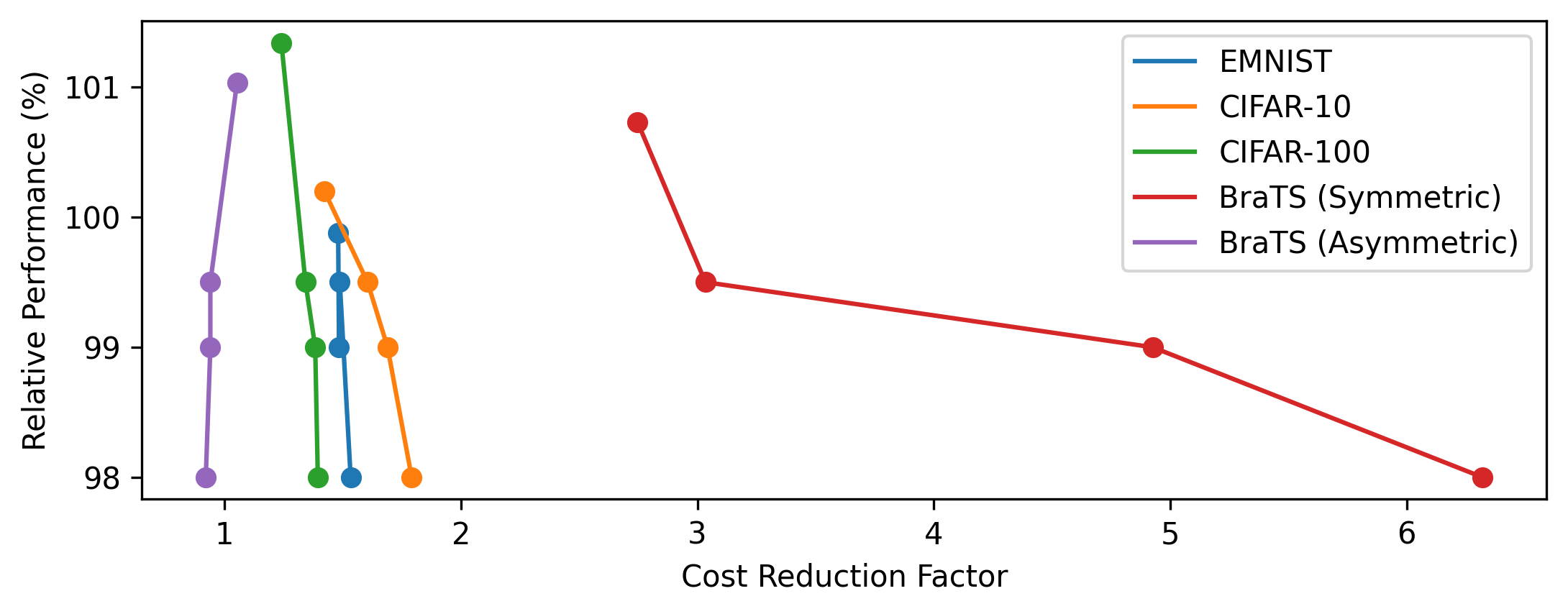}
    \vskip -0.1in
    \caption{Communication cost reduction at $98\%$, $99\%$, $99.95\%$, $\textit{best}$ compared to the baseline performance in the federated setting.} %
    \label{fig:fed-accer}
    \vskip -0.2in
\end{figure}

\begin{figure}[t]
\centering
\includegraphics[width=\linewidth]{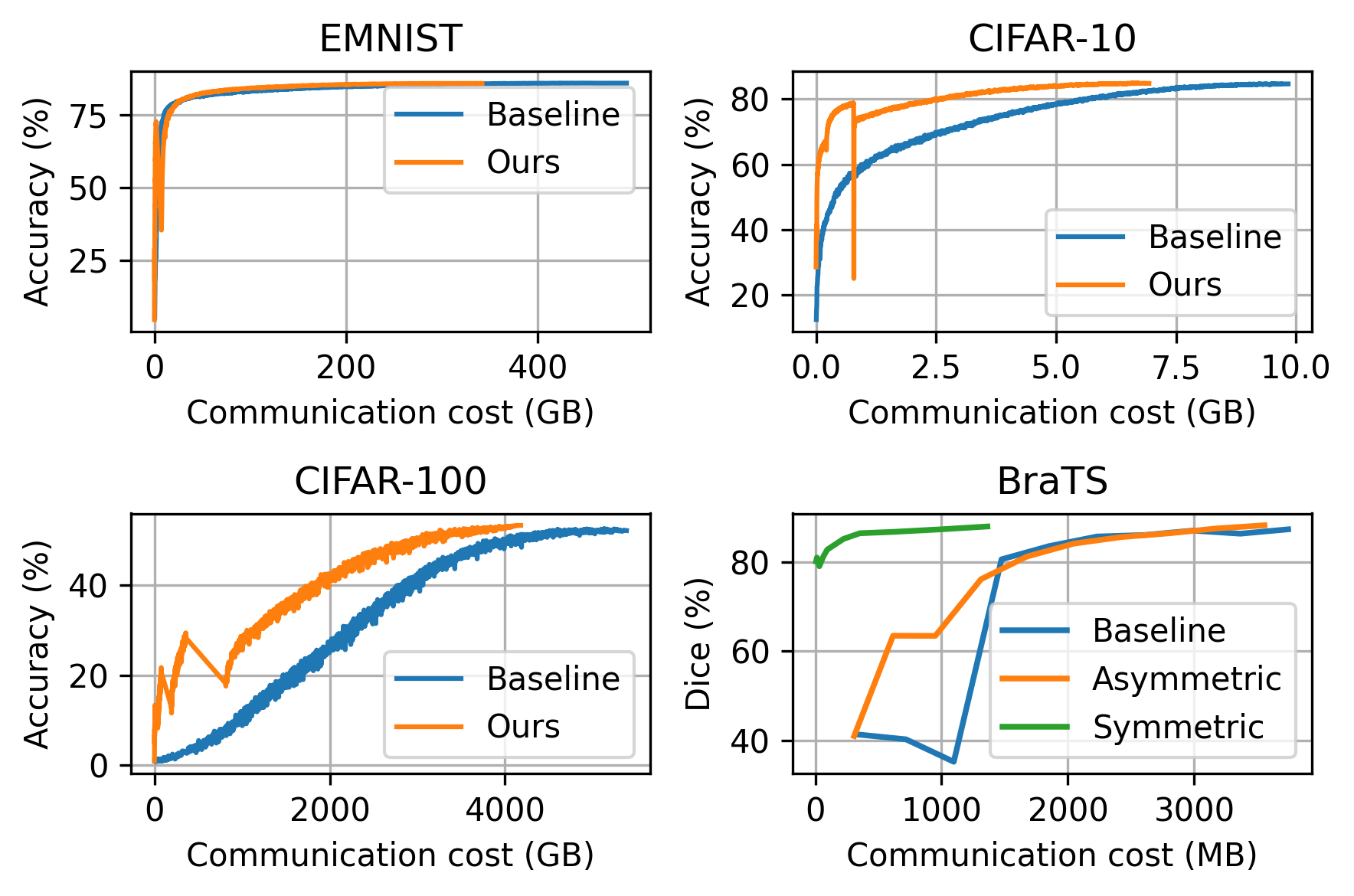}
\vskip -0.1in
\caption{Communication cost vs. Accuracy (\%) in federated settings on EMNIST (3400 clients, non-IID), CIFAR-10 (100 clients, IID), CIFAR-100 (500 clients, non-IID), and BraTS (10 clients, IID).}
\label{fig:fed-cost-acc}
\vskip -0.15in
\end{figure}

\subsection{Communication Efficiency}
\label{ssec:communication_efficiency}

We experiment in the federated setting to verify the communication efficiency of our method with FedAvg~\citep{mcmahan2017communication}. In particular, we consider classification tasks on three datasets, EMNIST, CIFAR-10, and CIFAR-100, and tumor segmentation tasks on the BraTS dataset. We follow the standard protocol as described in Section~\ref{ssec:setup} to train the models and average the results over three random seeds. Results in Table~\ref{table:fed-tasks} indicate that our method achieves comparable results on EMNIST and outperforms the baselines on all the other datasets. In addition, our method saves 20\% to 30\% two-way communication costs in classification and up to 63\% costs in segmentation. The result simultaneously confirms the effectiveness and efficiency of our method. We discuss the effect of different numbers of stages in Section~\ref{ssec:append_ablation_s}.\looseness=-1

\begin{table}[t]
\caption{Results in federated settings. We report accuracy (\%) for classification and Dice scores (\%) for segmentation, followed by cost reduction (CR) as compared to the baselines (end-to-end).}
\label{table:fed-tasks}
\vskip 0.1in
\centering
{\small
\begin{tabular}{@{}lccc@{}}
\toprule
 & Baseline & Ours & CR \\ \midrule
EMNIST & \textbf{85.75 $\pm$ 0.11} & 85.67 $\pm$ 0.06 & -29.49\% \\
CIFAR-10 & 84.67 $\pm$ 0.14 & \textbf{84.85 $\pm$ 0.30} & -29.70\% \\
CIFAR-100 & 52.08 $\pm$ 0.44 & \textbf{53.23 $\pm$ 0.09} & -22.90\% \\ \midrule
BraTS (Aym.) & 86.77 $\pm$ 0.45 & \textbf{87.66 $\pm$ 0.49} & -5.02\%  \\
BraTS (Sym.) & 86.77 $\pm$ 0.45& \textbf{87.96 $\pm$ 0.03} & -63.60 \% \\ \bottomrule
\end{tabular}
}
\vskip -0.1in
\end{table}

We compare the performance at different communication costs in Figure~\ref{fig:fed-cost-acc}. We observe that our method is communication-efficient over every cost budget, especially when the model parameters are not evenly distributed across sub-models. For instance, 3D-Unet has most of its parameters in the middle part of the model, making our \textit{Symmetric} update strategy extremely efficient. On the other hand, the \textit{Asymmetric} strategy shows marginal improvement since it starts from the heaviest portion of the model. The finding aligns with the motivation of progressive learning: learning from simpler models might facilitate training. We leave more analysis on segmentation in Section~\ref{ssec:append-vis-seg}. \looseness=-1

Lastly, we analyze the cost reduction when achieving $98\%$, $99\%$, $99.5\%$, and the best of the performance of the converged baseline. This experiment studies the behavior of our method when only granted limited budgets. Figure~\ref{fig:fed-accer} (and Figure~\ref{fig:append-fed-accer} in the appendix that displays a larger range) presents that except for the \textit{Asymmetric} strategy, our method improves communication efficiency across all datasets. 
In particular, it achieves practicable performance with 2x fewer costs in classification and up to 6.5x fewer costs in tumor segmentation. We also observe that the communication efficiency improves more when considering lower budgets. This property benefits when the time and communication budgets are limited~\citep{mcmahan2017communication}. \looseness=-1

\begin{table}[t]
\centering
\caption{Results of ProgFed with FedAvg, FedProx, and FedAdam on CIFAR-100 in the federated setting.}
\label{table:advanced_opt}
\vskip 0.1in
{\small
\begin{tabular}{@{}lccc@{}}
\toprule
\multicolumn{4}{c}{EMNIST} \\ \midrule
 & FedAvg & FedProx & FedAdam \\
End-to-end & \textbf{85.75} & \textbf{86.36} & \textbf{86.53} \\
FedProg (S=4) & 85.67 & 86.08 & 86.13 \\ \midrule
\multicolumn{4}{c}{CIFAR-100} \\ \midrule
 & FedAvg & FedProx & FedAdam \\
End-to-end & 52.08 & 53.25 & 56.21 \\
FedProg (S=4) & \textbf{53.23} & \textbf{54.28} & \textbf{60.55} \\ \bottomrule
\end{tabular}
}
\vskip -0.1in
\end{table}

\begin{table*}[tbp]
\caption{Federated ResNet-18 on CIFAR-100 with compression. LQ-X denotes linear quantization followed by used bits representing gradients, and SP-X denotes sparsification followed by the percentage of kept gradients (see Table~\ref{table:append-compress} for standard deviations).}
\label{table:compress}
\vskip 0.1in
\centering
{\small
\begin{tabular}{@{}lcccccccc@{}}
\toprule
 & Float & LQ-8 & LQ-4 & LQ-2 & SP-25 & SP-10 & \begin{tabular}[c]{@{}c@{}}LQ-8\\ +SP-25\end{tabular} & \begin{tabular}[c]{@{}c@{}}LQ-8\\ +SP-10\end{tabular} \\ \midrule
 & \multicolumn{8}{c}{Accuracy (\%)} \\ \cmidrule(l){2-9} 
Baseline & 52.08 & 49.40 & 49.55 & 47.26 & 51.23 & 51.79 & 49.67 & 50.25 \\
Ours & \textbf{53.23} & \textbf{53.07} & \textbf{52.32} & \textbf{52.87} & \textbf{52.00} & \textbf{51.86} & \textbf{52.19} & \textbf{52.24} \\ \midrule
 & \multicolumn{8}{c}{Compression Cost (\%)} \\ \cmidrule(l){2-9} 
Baseline & 100 & 25.00 & 12.50 & 6.25 & 25.00 & 10.00 & 6.25 & 2.50 \\
Ours & \textbf{77.10} & \textbf{19.28} & \textbf{9.64} & \textbf{4.82} & \textbf{19.28} & \textbf{7.71} & \textbf{4.82} & \textbf{1.93} \\ \bottomrule
\end{tabular}
}
\vskip -0.1in
\end{table*}

\begin{table*}[t]
\caption{Comparison between update strategies on CIFAR-100 with ResNet-18 in the centralized setting.}
\label{table:update_strategy}
\vskip 0.1in
\centering
{\small
\begin{tabular}{@{}lcccccc@{}}
\toprule
         & Baseline & Ours     & Layerwise & Partial  & Mixed     & Random   \\ \midrule
Accuracy (\%) & \textbf{76.08$\pm$0.12}    & 75.84$\pm$0.28    & 72.40$\pm$0.16     & 74.70$\pm$0.04    & 75.04$\pm$1.26     & 74.38$\pm$0.97    \\
Cost   & 1 & \textbf{0.86} & 1  & 1 & $\approx1$ & 0.88 \\ \bottomrule
\end{tabular}
}
\vskip -0.1in
\end{table*}

\begin{figure}[tbp]
    \centering
    \setlength{\tabcolsep}{1pt}
    \begin{tabular}{c}
        \includegraphics[width=0.9\linewidth]{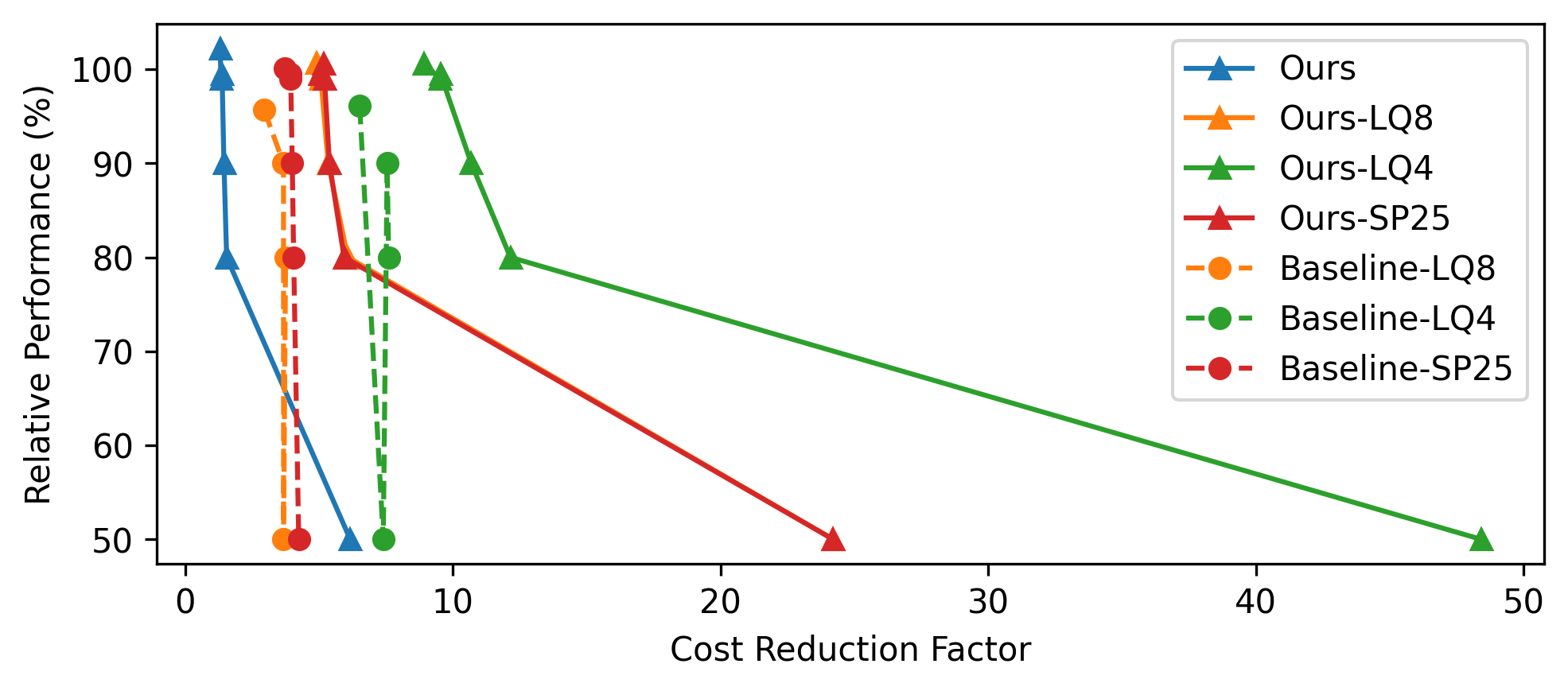} \\             
        (a) Moderate compression  \\ \includegraphics[width=0.9\linewidth]{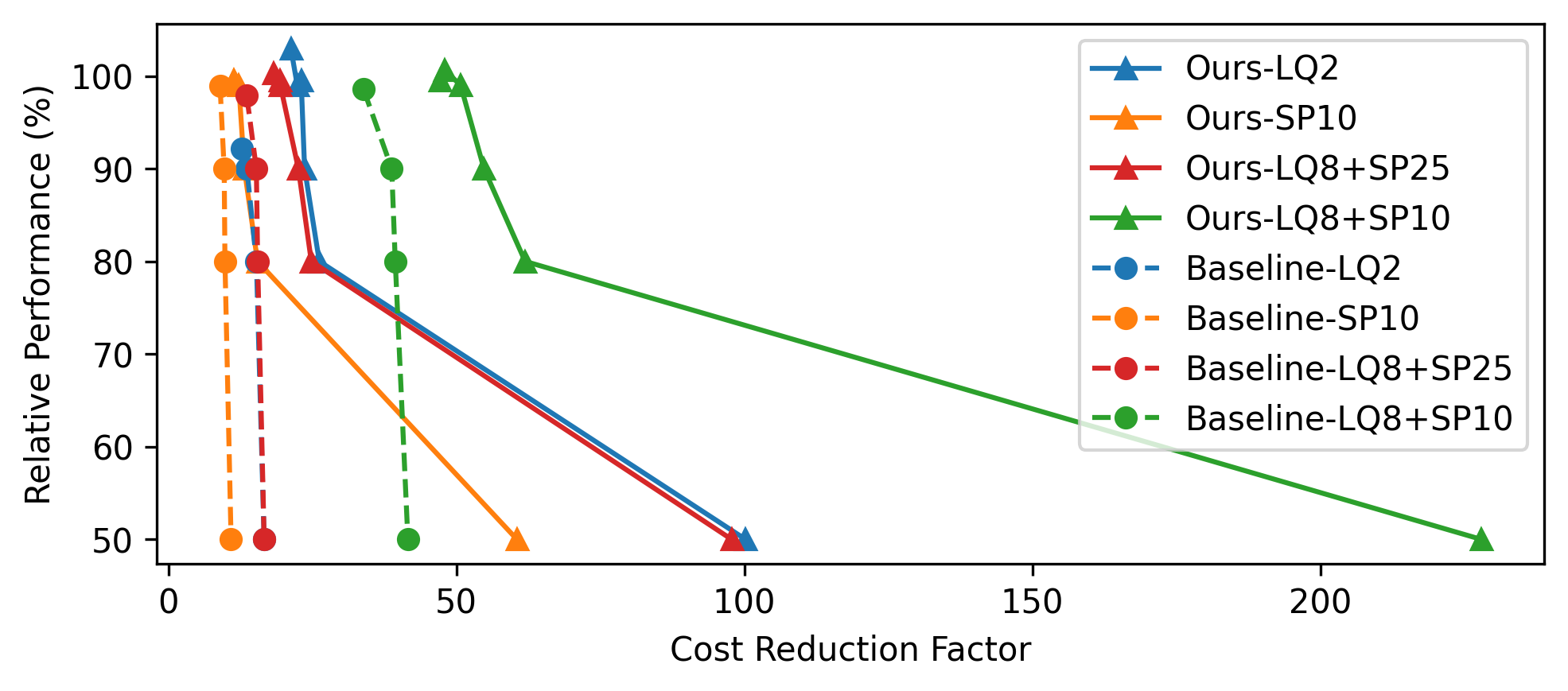} \\
       (b) Intensive compression 
    \end{tabular}
    \caption{Relative performance vs. communication cost reduction with federated ResNet-18 on CIFAR-100 with (a) modest compression and (b) intensive compression.}
    \label{fig:compression}
    \vskip -0.15in
\end{figure}

\subsection{Compatibility}
\label{ssec:compatibility}
\myparagraph{Advanced optimization} We show that ProgFed can generalize to federated optimizations beyond FedAvg, including FedProx~\citep{li2020federated} and FedAdam~\citep{reddi2021adaptive_fl}, on CIFAR-100 and EMNIST. FedProx mitigates the non-IID problem by penalizing the L2 distance between the updated client models and the global model. On the other hand, FedAdam approaches the problem by adopting an Adam optimizer on the server side. These methods impose client-side and server-side regularization on top of FedAvg. 

Results in Table~\ref{table:advanced_opt} show that our method works seamlessly with FedProx and FedAdam. We first observe that the performance improves on both datasets when applying FedProx and FedAdam. Similar to the result in Table~\ref{table:fed-tasks}, our method significantly outperforms the baseline on CIFAR-100 while performing comparably on EMNIST. The result verifies that ProgFed can be readily applied to advanced federated optimizations.

\myparagraph{Compression} We show that our method complements classical compression techniques including quantization and sparsification. We train several ResNet-18 on CIFAR-100 in the federated setting and apply linear quantization and sparsification following~\citet{konevcny2016federated}. Specifically, we consider 8 bits, 4 bits, and 2 bits for quantization (denoted by LQ-X), 25\% and 10\% for sparsification (denoted by SP-X), and their combinations. Table~\ref{table:compress} demonstrates the results. Our method clearly outperforms the baselines in all settings. It indicates that our method is more robust against compression errors, compatible with classical techniques, and thus permits a higher compression ratio. \looseness=-1

In addition, we visualize 50\%, 80\%, 90\%, 99\%, 99.5\%, and the best of the performance of the converged baseline against communication cost reduction in Figure~\ref{fig:compression}. We observe that our method is more efficient across all percentages in every pair (Ours vs.\ Baseline, plotted in the same color). Besides, the baseline fails to achieve comparable performance in many settings, e.g., the ones with quantization, while our method retains comparable performance even with high compression ratios.
Interestingly, even with additional compression, our method still facilitates learning at earlier stages. For example, Ours-LQ8+SP25 achieves comparable performance around 50x faster than the baseline, 60x faster to achieve 80\%, and more than 200x faster to achieve 50\% of performance. Overall, these properties grant our method to adequately approach limited network bandwidth and open up the possibility of more advanced compression schemes. \looseness=-1

\subsection{Analysis of ProgFed}
\label{ssec:analysis_progfed}
\myparagraph{Effect of norm discrepancy} As discussed in Section~\ref{ssec:thm}, the convergence rate of the full model is controlled by norm discrepancy, namely $q$. As $q_t$ approaches 1, the convergence rate will be closer to the convergence speed of the sub-models. We empirically evaluate the norm discrepancy on CIFAR-100 with ResNet-18 in the centralized setting. 
Figure~\ref{fig:gradient_ratio} shows that the norm discrepancy decreases as the sub-models gradually recover the full model. It suggests that spending too much time on earlier stages may hurt the convergence speed while offering a higher compression ratio. This outlines the trade-off between communication and training efficiency.  

\begin{figure}[tbp]
    \centering
    \includegraphics[width=\linewidth]{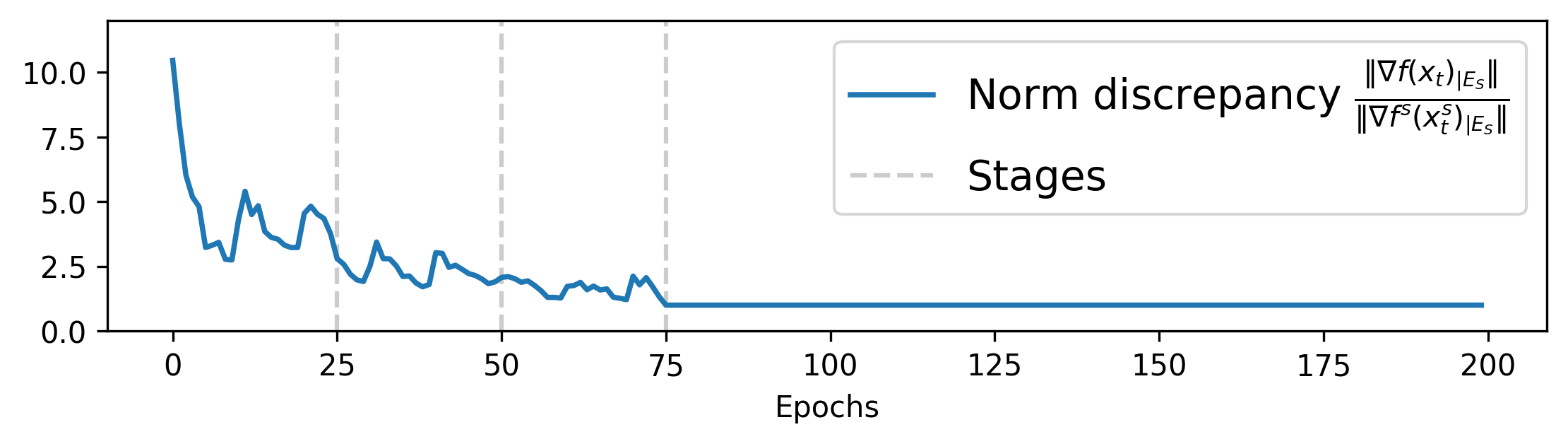}
    \vskip -0.15in
    \caption{Norm discrepancy.}
    \label{fig:gradient_ratio}
    \vskip -0.2in
\end{figure}

\myparagraph{Comparison between update strategies} As described in Section~\ref{sec:method}, ProgFed progressively trains the network from the shallowest sub-model $\mathcal{M}^1$ to the full model $\mathcal{M}$. We verify our update strategy by comparing it with various baselines in the centralized setting. \qcr{Baseline}: end-to-end training; \qcr{Layerwise} only updates the latest layer $E_i$ while still passing the input through the whole model $\mathcal{M}$; \qcr{Partial} partially updates $E_s$ but acquires supervision from the last head $G_S$; \qcr{Mixed} combines \qcr{Partial} and \qcr{Ours}, trained on supervision from both $G_i$ and $G_S$; \qcr{Random} randomly chooses a sub-model $\mathcal{M}^s$ to update, rather than follows progressive learning. 

Table~\ref{table:update_strategy} presents the performance and the computation cost ratio. We make several crucial observations. (1) \qcr{Ours} and \qcr{Random} do not pass the input through the whole network, making them consume fewer computation costs. (2) \qcr{Layerwise} greedily trains the network but achieves the worst performance, which highlights the importance of end-to-end fine-tuning. (3) \qcr{Ours} outperforms \qcr{Random} in both costs and accuracy, verifying the necessity of progressive learning. We also note that our method does not require additional memory space (compared to \qcr{Random}) and is easy to implement.

\huipo{\section{Discussion}
\label{sec:discussion}
There are two important hyperparameters $S$ and $\{T_s\}_{s=1}^{S}$ in ProgFed (Sec.~\ref{ssec:proposed_method}). We note that hyperparameter selection in FL remains an open problem and could severely affect performance. For instance, \citet{reddi2021adaptive_fl} show that FL models are sensitive to learning rates. However, we show in Section~\ref{ssec:append_ablation_s} that ProgFed works smoothly with various $s$ at the cost of slightly more epochs while still taking fewer costs to match the baseline performance. This aligns with Theorem~\ref{thm:convergence} that ProgFed may take more epochs to converge but consume much fewer per-iteration costs. Lastly, we note that Theorem~\ref{thm:convergence} is general for partial model updates. We leave incorporating structural updates and the non-IID data assumption for future work. 

In addition to the efficiency and efficacy of ProgFed, we observe that the area under the curves of our method in Figure~\ref{fig:acc-flops} and \ref{fig:fed-cost-acc} are always larger than the end-to-end baselines. It suggests that our approach can provide more practical utility at all times. It aligns with the notion \emph{anytime learning}~\citep{caccia2021anytime}, where the models are expected to provide the best utility at any time during training. This feature is favorable in practice since users may access the system at any time of training, and ProgFed demonstrates great potential to implement such systems.
}
\section{Conclusion}
\label{sec:conclusion}
Beyond prior work on expressing models in compact formats, we show a novel approach to modifying the training pipeline to reduce the training costs. We propose ProgFed and show that progressive learning can be seamlessly applied to federated learning for communication and computation cost reduction. Extensive results on different architectures from small CNNs to U-Net and different tasks from simple classification to medical image segmentation show that our method is communication- and computation-efficient, especially when the training budgets are limited. Interestingly, we found that a progressive learning scheme has even led to improved performance in vanilla learning and more robust results when learning is perturbed e.g. in the case of gradient compression, which highlights progressive learning as a promising technique in itself with future application domains such as privacy-preserving learning, advanced compression schemes, and strong anytime-learning performance.

\section*{Acknowledgment}
This work was partially funded by the Helmholtz Association within the project ”Trustworthy Federated
Data Analytics (TFDA)” (ZT-I-OO1 4) and supported by the Helmholtz Association's Initiative and Networking Fund on the HAICORE@FZJ partition. 
\bibliography{example_paper}
\bibliographystyle{icml2022}

\newpage
\appendix
\onecolumn
\section{Appendix}
\subsection{Proof of Theorem~\ref{thm:convergence}}
In this section we prove Theorem~\ref{thm:convergence}. The proof builds on~\citep{mohtashami2021simultaneous} that considered training of subnetworks, but not the progressive learning case.

\begin{lemma}
\label{lemma_onestep}
Let $\xx_t$ denote the weights of the full model, and $\xx_t^s$ the weights of the model that is active in iteration $t$. Note that it holds $\xx_{t \mid E_s} = \xx^s_{t \mid E_s}$ as per the definition in the main text. It holds,
\begin{equation}
    \mathds{E}f(\mathbf{x}_{t+1}) \leq f(\mathbf{x}_t) - \frac{\gamma}{2}\alpha^2_t\norm{\nabla f^s(\mathbf{x}_t^s)_{\mid E_s}}^2+\frac{\gamma^2L}{2}\sigma^2
\end{equation}
\end{lemma}

\begin{proof}
Let's abbreviate $\mathbf{g}^s_t = g^s(\xx_t^s)$.
By the update equation $\xx_{t+1}^s = \xx_t^s - \gamma_t \mathbf{g}^s_t$ it holds $\xx_{t+1 \mid E_s} = \xx_{t \mid E_s} - \gamma_t \mathbf{g}^s_{t \mid E_s}$. With the $L$-smoothness assumption and the definition of $\alpha_t$,
\begin{equation*}
\begin{split}
    \mathds{E}f(\mathbf{x}_{t+1}) &\leq f(\mathbf{x}_t) - \gamma_t\langle \nabla f(\mathbf{x}_t)_{\mid E_s}, \mathds{E}[\mathbf{g}^s_t]_{\mid E_s}\rangle + \frac{\gamma^2_tL}{2}\mathds{E}\norm{\mathbf{g}^s_{t\mid E_s}}^2 \\
    & = f(\mathbf{x}_t) - \gamma_t\langle \nabla f(\mathbf{x}_t)_{\mid E_s}, \mathds{E}[\mathbf{g}^s_t]_{\mid E_s}\rangle + \frac{\gamma^2_tL}{2}\mathds{E}(\norm{\mathbf{g}^s_{t \mid E_s} -\mathds{E}\mathbf{g}^s_{t \mid E_s}}^2+\norm{\mathds{E}\mathbf{g}^s_{t \mid E_s}}^2) \\
    & \leq f(\mathbf{x}_t) - \gamma_t\langle \nabla f(\mathbf{x}_t)_{\mid E_s}, \nabla f^s(\xx_t^s)_{\mid E_s} \rangle + \frac{\gamma^2_tL}{2}\mathds{E}(\norm{\mathbf{g}^s-\mathds{E}\mathbf{g}^s_t}^2 +\norm{\nabla f^s(\mathbf{x}_t^s)_{\mid E_s}}^2) \\
    & \leq f(\mathbf{x}_t) - \gamma_t\langle \nabla f(\mathbf{x}_t)_{\mid E_s}, \nabla f^s(\xx_t^s)_{\mid E_s} \rangle + \frac{\gamma^2_tL}{2}\norm{\nabla f^s(\mathbf{x}_t^s)_{\mid E_s}}^2+\frac{\gamma^2_{t}L}{2}\sigma^2 \\
    & \leq f(\mathbf{x}_t) - \gamma_t\alpha_t(1-\frac{\gamma_t}{2\alpha_t}L)\norm{\nabla f^s(\mathbf{x}_t^s)_{\mid E_s}}^2 +\frac{\gamma^2_{t}L}{2}\sigma^2 \\
    & \leq f(\mathbf{x}_t) - \frac{\gamma_t}{2}\alpha_t\norm{\nabla f^s(\mathbf{x}_t^s)_{\mid E_s}}^2 +\frac{\gamma^2_{t}L}{2}\sigma^2 \\
    & \leq f(\mathbf{x}_t) - \frac{\gamma}{2}\alpha_t^2\norm{ \nabla f^s(\mathbf{x}_t^s)_{\mid E_s}}^2 +\frac{\gamma^2 L}{2}\sigma^2
\end{split}
\end{equation*}
Where in the last equation we used the facts that $\alpha_t \leq 1$ and $\gamma_t=\alpha_t\gamma$. 
\end{proof}

We now prove Theorem~\ref{thm:convergence}.

\begin{proof}
We first define $F_t \vcentcolon= \mathds{E}f(\mathbf{x}_t) - (\min_{\xx} f(\xx))$. By rearranging Lemma~\ref{lemma_onestep}, we have
\begin{equation}
    \frac{1}{2}\mathds{E}\alpha_t^2\norm{\nabla f^s(\mathbf{x}_t^s)_{\mid E_s}}^2 \leq \frac{F_t-F_{t+1}}{\gamma} + \frac{\gamma L}{2}\sigma^2.
\end{equation}
Next, with telescoping summation, we have
\begin{equation}
\label{eq_fixed_mask}
    \frac{1}{T}\sum_{t=0}^{T-1}\mathds{E}\alpha_t^2\norm{\nabla f^s(\mathbf{x}_t^s)_{\mid E_s}}^2 \leq \frac{2(F_0-F_{T-1})}{T\gamma} + \gamma L\sigma^2 \leq \frac{2F_0}{T\gamma} + \gamma L\sigma^2
\end{equation}

We now can prove the first of part Theorem~\ref{thm:convergence} by setting the step size $\gamma$ to be $\mathcal{O}(\min\{ \frac{1}{L}, (\frac{F_0}{\sigma^2T})^\frac{1}{2}\}$ as in \citep{mohtashami2021simultaneous}.

To prove the convergence of the model of interest (the second part),
\begin{equation}
\begin{split}
    \frac{1}{T}\sum_{t=0}^{T-1}\norm{\nabla f(\mathbf{x}_t)}^2 &= \frac{1}{T}\sum_{t=0}^{T-1}\frac{\norm{\nabla f(\mathbf{x}_t)}^2}{\alpha^2_t\norm{\nabla f^s(\mathbf{x}_t^s)_{\mid E_s}}^2} \alpha_t^2 \norm{\nabla f^s(\mathbf{x}^s_t)_{\mid E_s}}^2 \\
    &= \frac{1}{T}\sum_{t=0}^{T-1}q_t^2\alpha^2_t\norm{\nabla f^s(\mathbf{x}_t^s)_{\mid E_s}}^2 \leq q^2\frac{1}{T}\sum_{t=0}^{T-1}\alpha^2_t\norm{\nabla f^s(\mathbf{x}_t^s)_{\mid E_s}}^2 
\end{split}
\end{equation}
where 
\begin{equation}
    q_t = \frac{\norm{\nabla f(\mathbf{x}_t)}}{\alpha_t\norm{\nabla f^s(\mathbf{x}_t^s)_{\mid E_s}}} \qquad\text{and} \qquad
    q = \max_{t\in[T]} q_t.
\end{equation}
By definition $q \geq q_t$ for all $t \in [T]$, we reach the last inequality and combine it with the first part of the theorem. 
\begin{equation}
    \frac{1}{T}\sum_{t=0}^{T-1}\alpha^2_t\norm{\nabla f^s(\mathbf{x}_t)}^2 \leq \frac{\epsilon}{q^2} 
\end{equation}
\end{proof}

Using $\frac{\epsilon}{q^2}$ as the new threshold, we immediately prove the second part.

\section{Implementation details}
\label{sec:append-implementation}
We describe details of the datasets used in Section~\ref{sec:exp} and present the hyper-parameters in Table~\ref{table:append_fed_setting}.

\begin{table*}[tb]
\caption{Parameters for federated experiments}
\label{table:append_fed_setting}
\vskip 0.1in
\centering
\begin{tabular}{@{}lcccc@{}}
\toprule
Dataset & \#clients & \#clients\_per\_epoch & batch\_size & \#epochs \\ \cmidrule(l){2-5} 
EMNIST & 3400 & 68 & 20 & 1500 \\
CIFAR-10 & 100 & 10 & 50 & 2000 \\
CIFAR-100 & 500 & 40 & 20 & 3000 \\
BraTS & 10 & 10 & 3 & 100 \\ \midrule
 & \#epoch\_per\_client & \#stages ($S$) & $T_s$ & \#epochs\_for\_warmup \\ \cmidrule(l){2-5} 
EMNIST & 1 & 3 & 250 & 5 \\
CIFAR-10 & 5 & 4 & 250 & 0 \\
CIFAR-100 & 1 & 4 & 375 & 25 \\
BraTS & 3 & 4 & 25 & 0 \\ \bottomrule
\end{tabular}
\vskip -0.1in
\end{table*}

\myparagraph{CIFAR-10}
We conduct experiments on CIFAR-10 datasets for federated learning, following the setup of previous work~\citep{mcmahan2017communication}. The dataset is divided into 100 clients randomly, namely iid distributions for every client. We adopt the same CNN architecture with 122,570 parameters. 

\myparagraph{CIFAR-100} 
We follow the federated learning benchmark of CIFAR-100 proposed in~\citep{reddi2021adaptive_fl} to conduct the experiments on CIFAR-100. We use ResNet-18/-152 (batch norm are replaced with group norm~\citep{hsieh2020non}) and VGG-16/-19 in the centralized setting, while only considering ResNet-18 in the federated experiments. This setup allows us to evaluate the federated learning systems on non-IID distributions, where we use the splits as suggested in~\citep{reddi2021adaptive_fl}.

\myparagraph{EMNIST}
We follow the benchmark setting in~\citep{reddi2021adaptive_fl} to experiment. There are 3,400 clients and  671,585 training examples distributed in a non-iid fashion. The models are eventually evaluated on 77,483 examples, resulting in a challenging task.

\myparagraph{BraTS}
In addition to image classification, we conduct experiments on brain tumor segmentation based on~\citep{sheller2020multi_brain_fl}. We train a 3D-Unet on the BraTS2018 dataset, which includes 285 MRI scans annotated by five classes of labels. The network has 9,451,567 parameters. The training set is randomly partitioned into ten clients. All clients participate in every training round and locally train their models for three local epochs. This setting matches the practical medical applications. Institutions often own relatively stable network conditions, and the data are rare and of high resolution.

\myparagraph{Architectures}
ConvNets for EMNIST and MNIST consist of two convolution layers, termed Conv1 and Conv2, followed by two fully connected layers, termed FC1 and FC2. To apply progressive learning with $S=3$, we set Conv1, Conv2, FC1 to be the three stages, namely $E_i$, and FC2 to the final head, namely $G_S$. As for VGGs, we divide the whole networks into five components according to the max-pooling layers. We combine the first two to be $E_1$ and set the others to be the remaining $E_i$ under the setting $S=4$. To apply ProgFed to ResNets~\cite{he2016deep}, we first replace the batch normalization layers with group normalization. By convention, ResNets have five convolution components, i.e. Conv1, Conv2\_x, Conv3\_x, Conv4\_x, and Conv5\_x. We combine Conv1 and Conv2\_x to be $E_1$ and all the other components to be the remaining $E_i$. It thus matches $S=4$ in our setting.

\section{More Results}
\label{sec:append-more-results}

\subsection{Comparison between update strategies}
As described in Section~\ref{ssec:analysis_progfed}, we compare ProgFed to other baselines. We additionally report the performance vs. computation costs and performance vs. epochs in Figure~\ref{fig:append_ablation}, where \qcr{Ours} reaches comparable performance while consuming the least cost.

\begin{figure}[tb]
    \centering
    \includegraphics[width=0.8\linewidth]{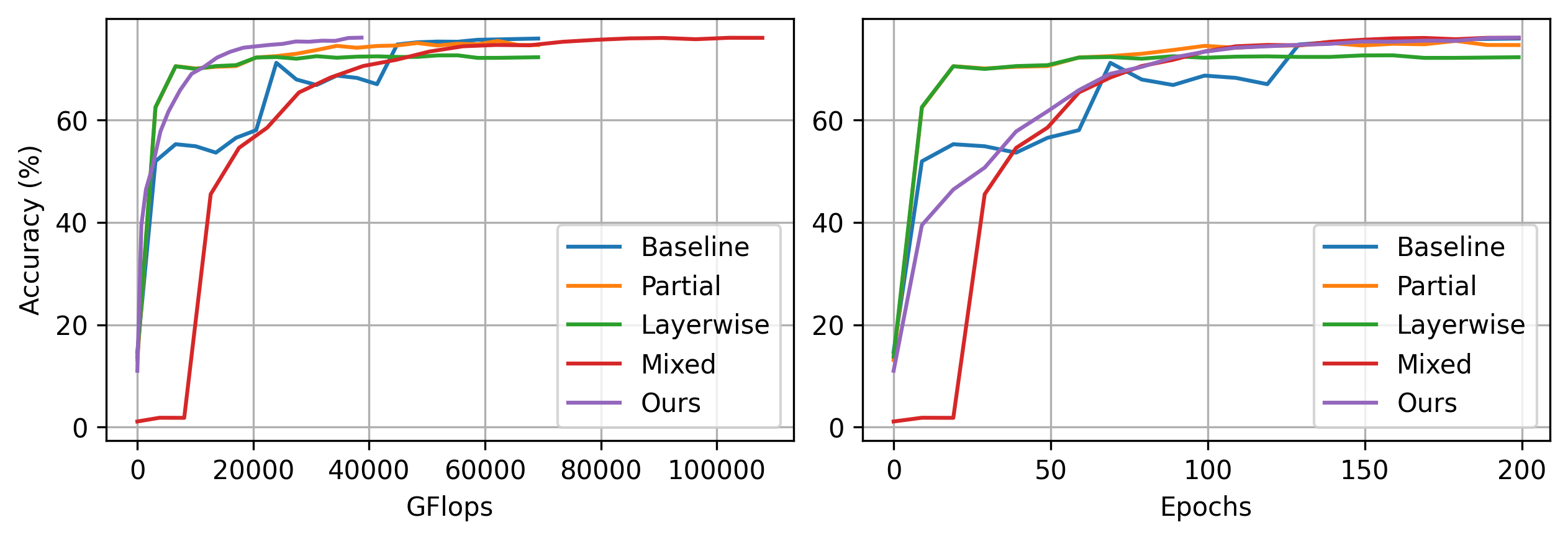}
    \caption{Performance vs. computation costs and Performance vs. epochs when comparing our method to different updating strategies.}
    \label{fig:append_ablation}
\end{figure}

\begin{figure}[tbp]
    \centering
    \includegraphics[width=\linewidth]{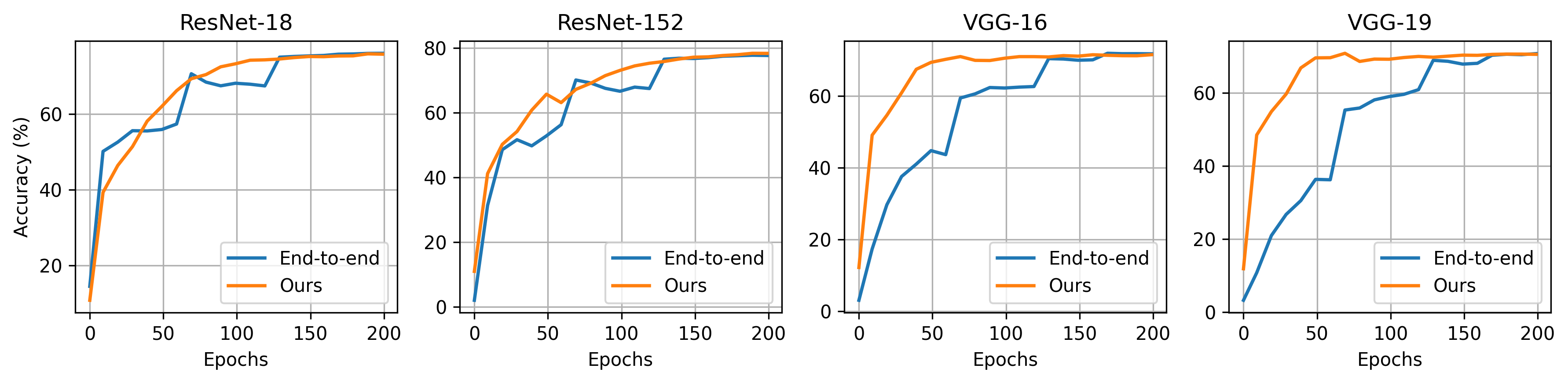}
    \caption{Accuracy (\%) vs. Epochs on CIFAR-100 in the centralized setting.}
    \label{fig:append-cen-clf-epochs}
\end{figure}

\begin{figure}[tbp]
    \centering
    \includegraphics[width=\linewidth]{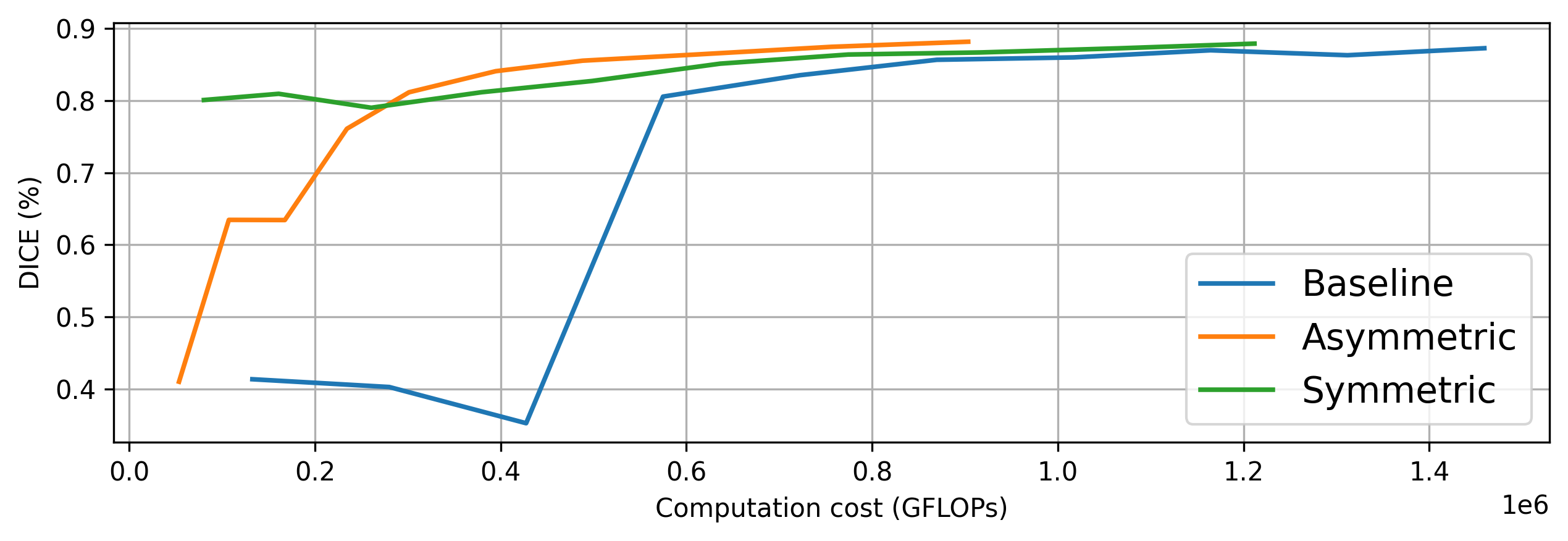}
    \caption{DICE (\%) vs. computation costs on BraTS.}
    \label{fig:append-fed-seg-dice-flop}
\end{figure}

\begin{figure}[tbp]
    \centering
    \includegraphics[width=\linewidth]{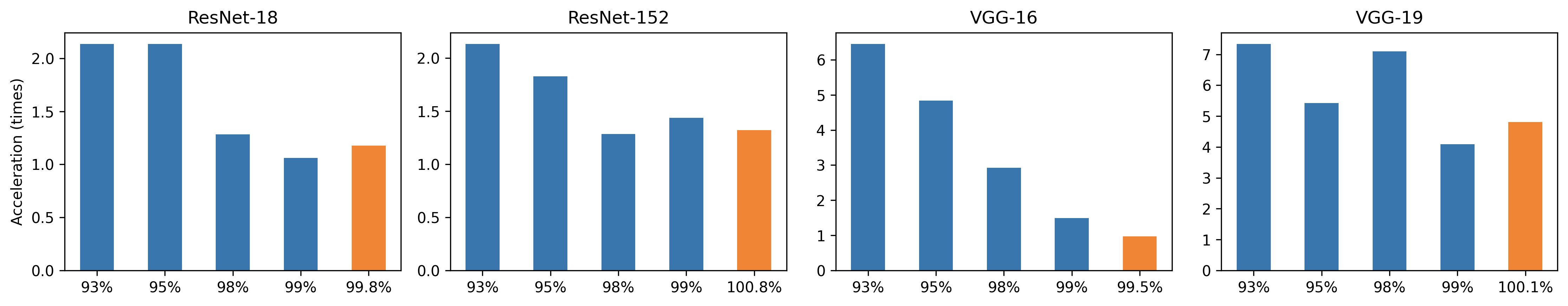}
    \caption{Computation acceleration at different percentage of performance. The orange bar indicates the best performance of our method.}
    \label{fig:append-cen-clf-accer-bar}
\end{figure}

\begin{figure}[tbp]
\begin{minipage}[t]{0.48\linewidth}
    \centering
    \includegraphics[width=\linewidth]{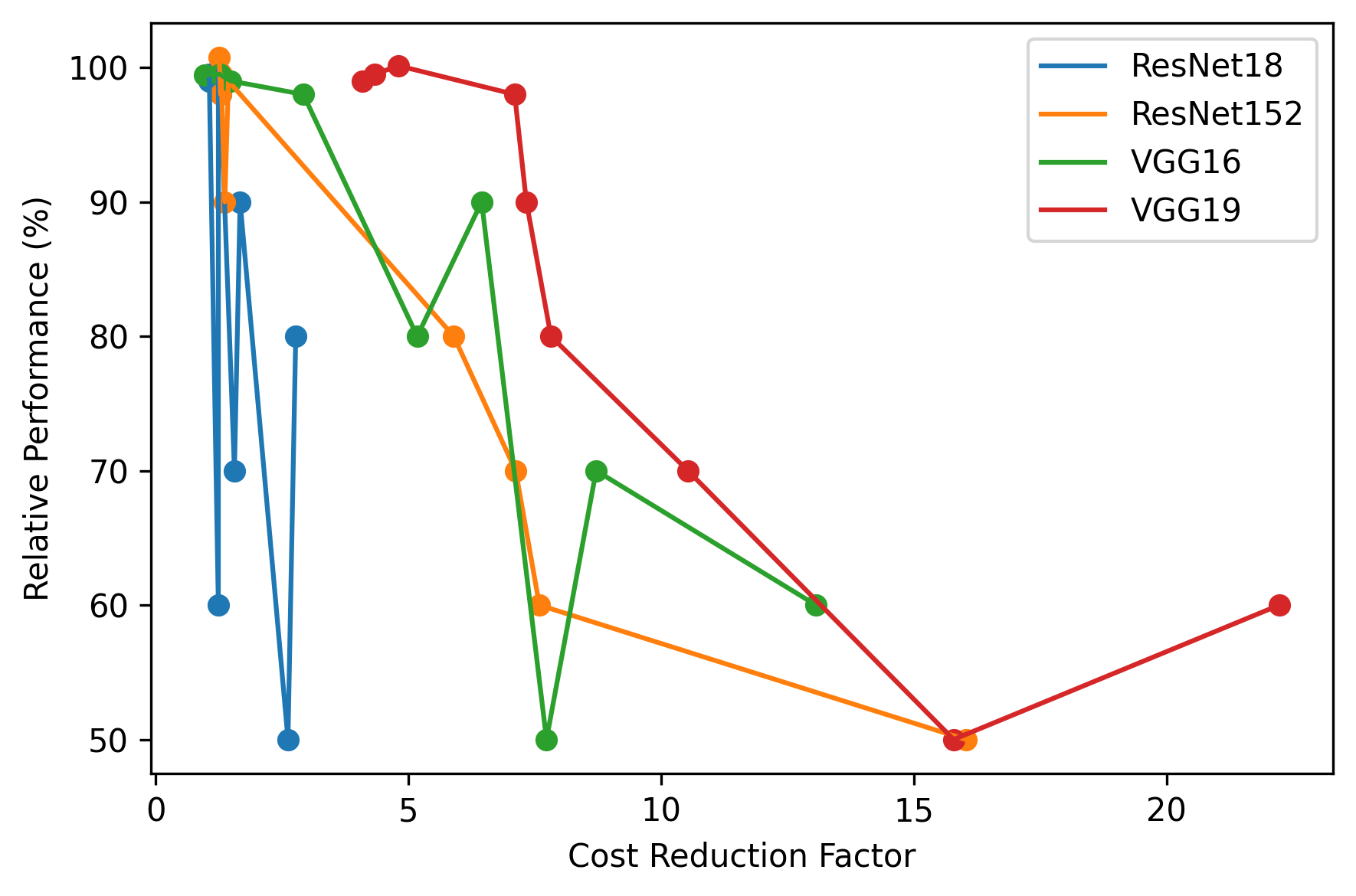}
    \caption{Computation cost reduction at $\{50\%$, $60\%$, $70\%$, $80\%$, $90\%$ $98\%$, $99\%$, $99.95\%$, $\textit{best}\}$ of the baseline performance in the centralized setting.} 
    \label{fig:append-cen-accer}
\end{minipage}
\begin{minipage}{0.02\linewidth}
\end{minipage}
\begin{minipage}[t]{0.48\linewidth}
    \centering
    \includegraphics[width=\linewidth]{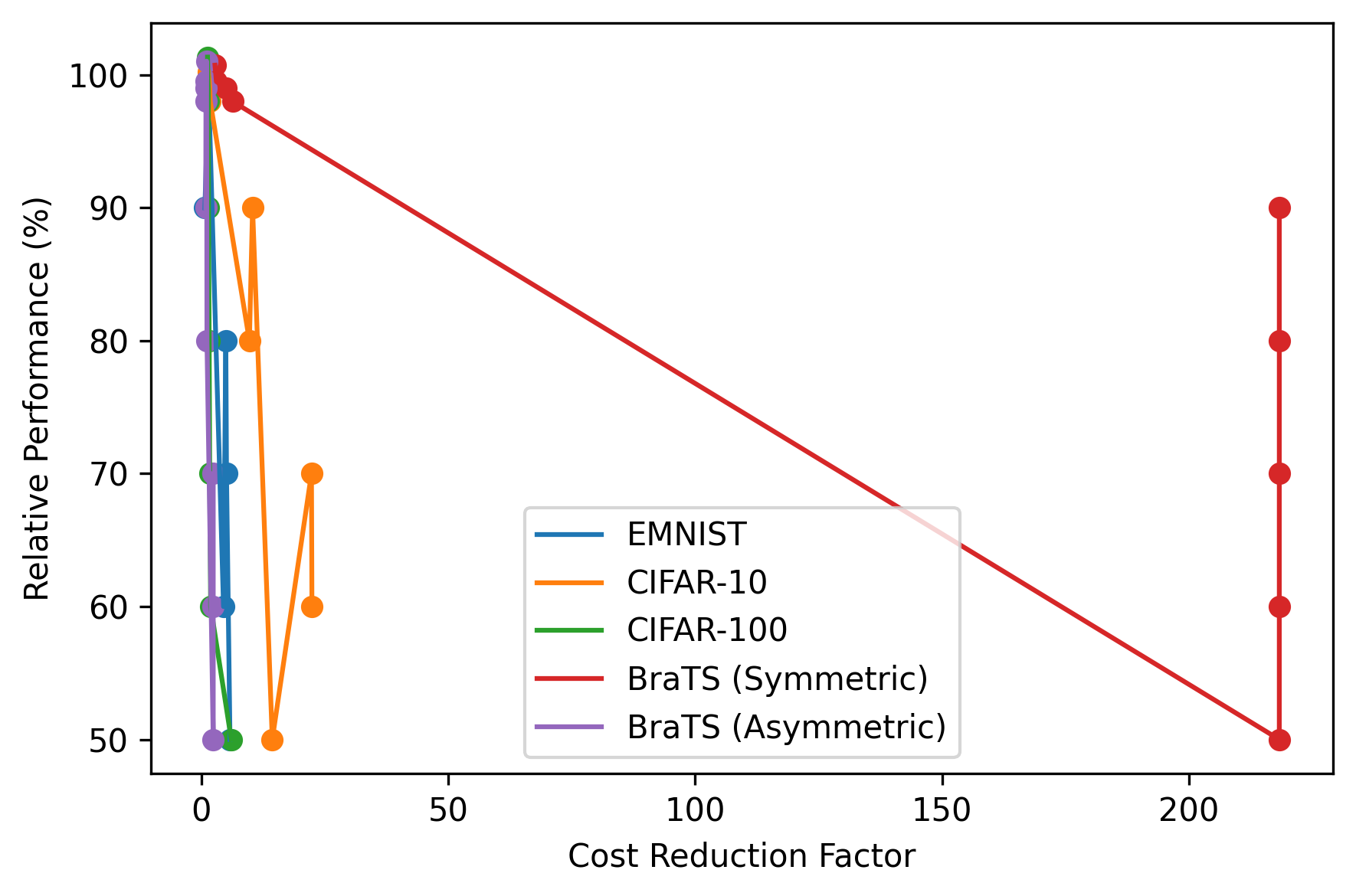}
    \caption{Communication cost reduction at $\{50\%$, $60\%$, $70\%$, $80\%$, $90\%$ $98\%$, $99\%$, $99.95\%$, $\textit{best}\}$ of the baseline performance in the federated setting.} 
    \label{fig:append-fed-accer}
\end{minipage}
\end{figure}

\subsection{Computation Efficiency}
We present more experiments in the centralized setting to prove the computation efficiency of our method.
Figure~\ref{fig:append-cen-clf-epochs} presents accuracy vs. epochs with four architectures on CIFAR-100. The result indicates that our method converges comparably faster to end-to-end training in practice.
Figure~\ref{fig:append-cen-clf-accer-bar} presents Figure~\ref{fig:cen-accer} in bar charts. Similar to Figure~\ref{fig:cen-accer}, our method improves across architectures while VGGs benefit even more from our method. Figure~\ref{fig:append-fed-seg-dice-flop} presents the computation costs of 3D-Unets on the BraTS dataset. We make the first observation that tumor segmentation requires heavy computation. Interestingly, even though the earlier stages of \textit{Symmetric} consume much fewer communication costs (Figure~\ref{fig:fed-cost-acc}), they require more computation costs than \textit{Asymmetric}. It might root from the higher resolution of feature maps that \textit{Symmetric} keeps and thus lead to a trade-off between communication and computation costs.

Figure~\ref{fig:append-cen-clf-accer-bar} extend Figure~\ref{fig:cen-accer} to a larger range $\{50\%$, $60\%$, $70\%$, $80\%$, $90\%$ $98\%$, $99\%$, $99.95\%$, $\textit{best}\}$. The result shows that our method benefits across models and is especially efficient when training budgets are limited.

\begin{figure}[tb]
\setlength{\tabcolsep}{1pt}
\begin{tabular}{cc}
\includegraphics[width=0.49\linewidth]{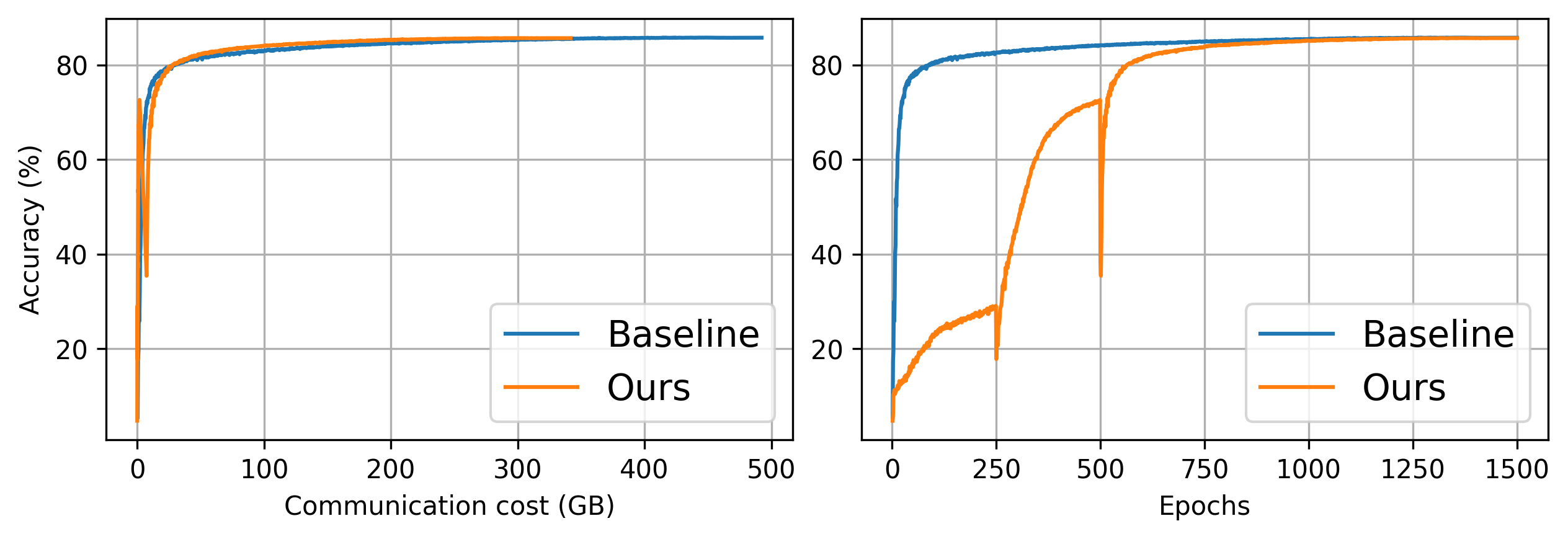} &
\includegraphics[width=0.49\linewidth]{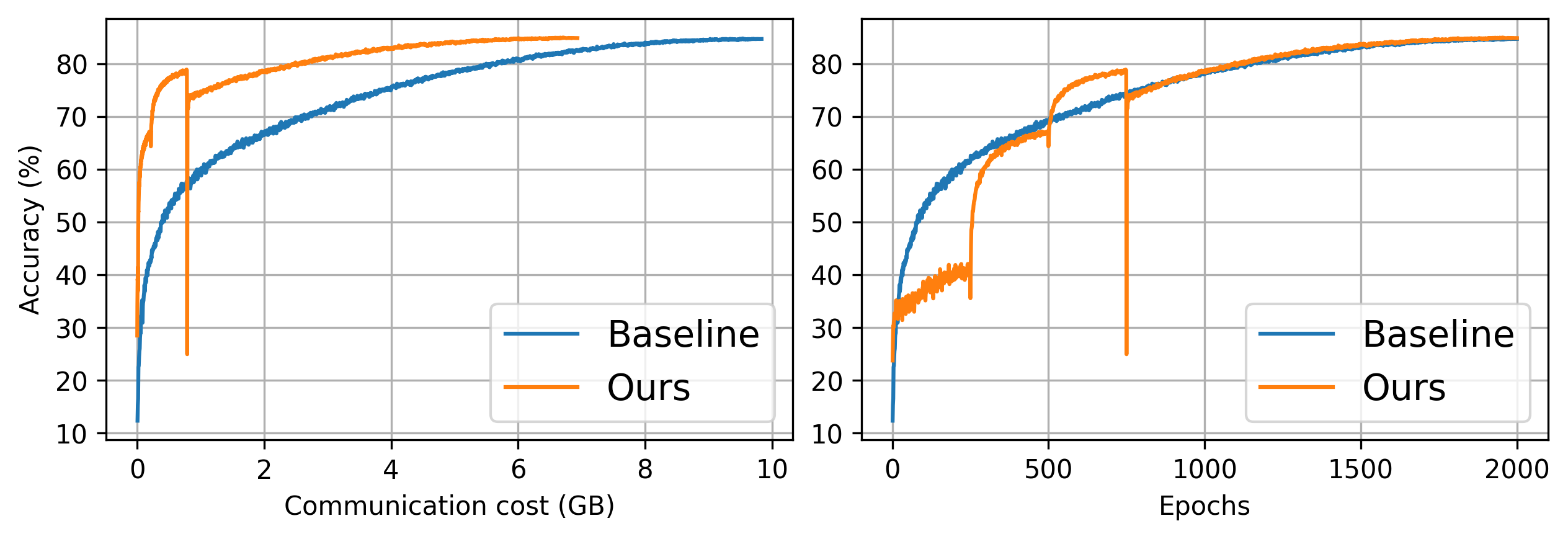} \\
(a) ConvNet on EMNIST & (b) ConvNet on CIFAR-10 \\ 
\includegraphics[width=0.49\linewidth]{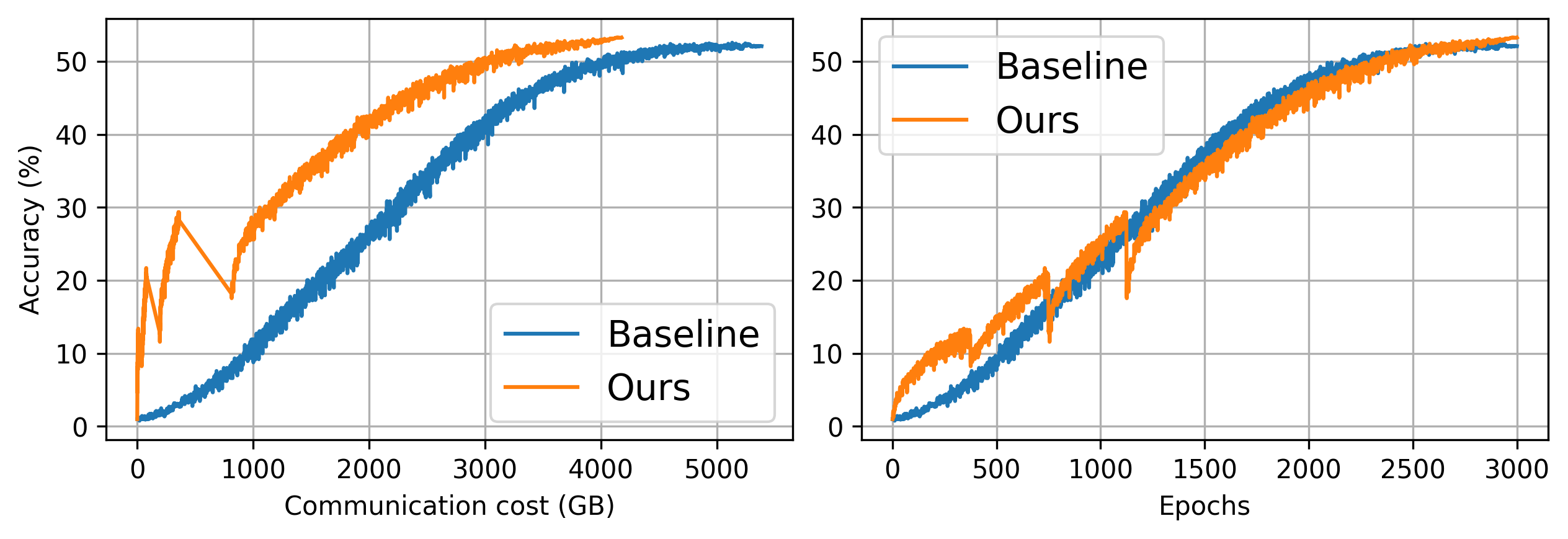} & \includegraphics[width=0.49\linewidth]{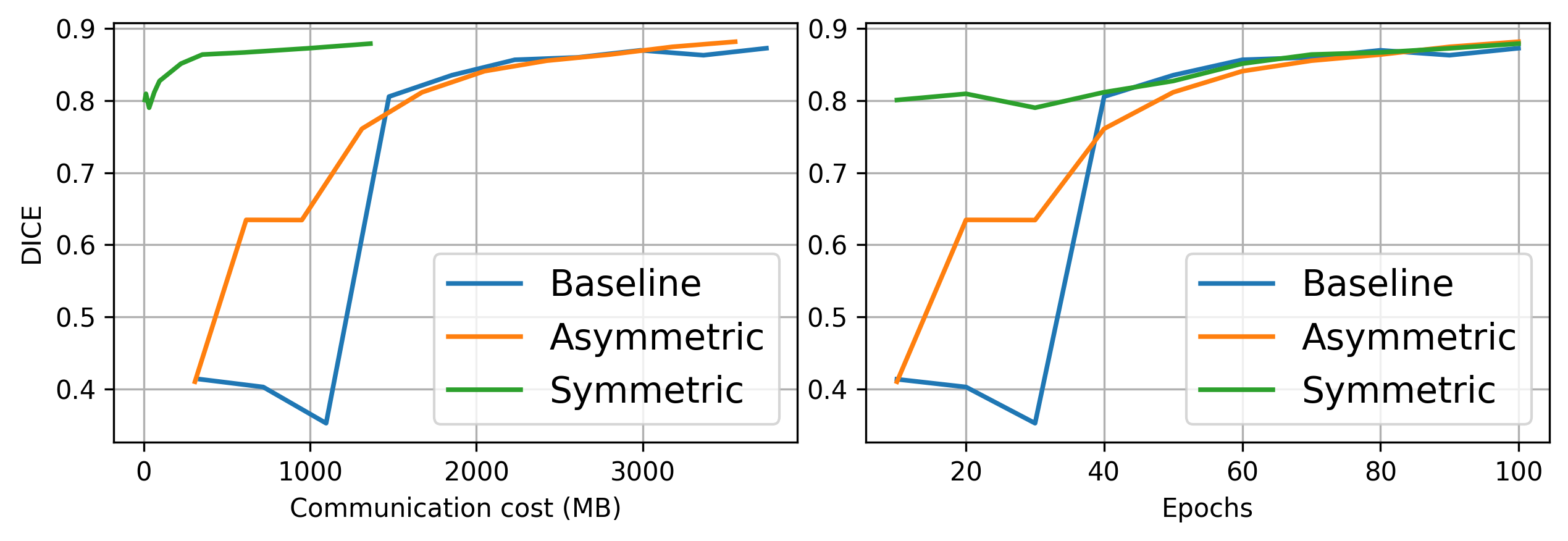}\\
(c) ResNet-18 on CIFAR-100 & (d) 3D-Unet on BraTS
\end{tabular}
\caption{Accuracy vs. computation costs and accuracy vs. epochs in the federated setting. (a)(b)(c) shows the result for three classification tasks; (d) shows the result for the segmentation task, where two update strategies \textit{Symmetric} and \textit{Asymmetric} are adopted for 3D-Unet.}
\label{fig:append-fed-tasks}
\end{figure}

\begin{figure}[tb]
    \centering
    \includegraphics[width=\linewidth]{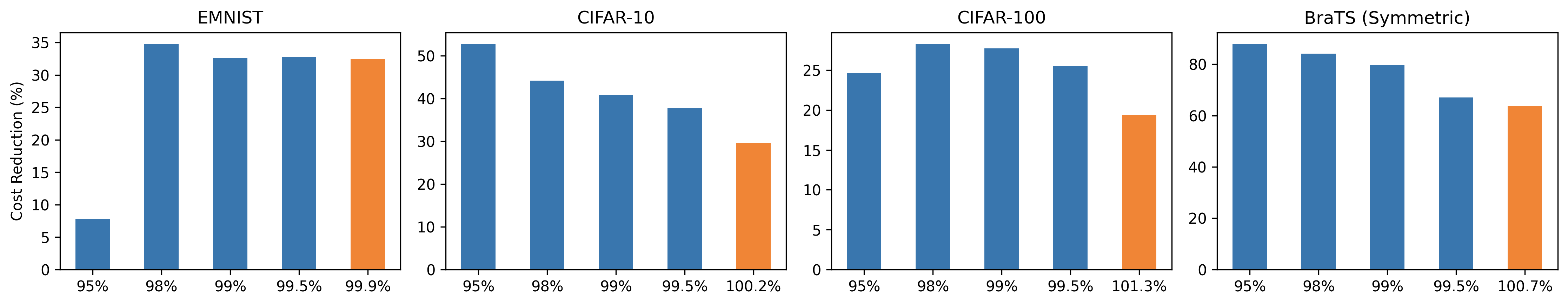}
    \caption{Communication cost reduction at different percentage of performance. The orange bar indicates the best performance of our method.}
    \label{fig:append-fed-accer-bar}
\end{figure}

\subsection{Communication Efficiency}
We present more experiments in the federated setting to prove the communication efficiency of our method.
To complement Figure~\ref{fig:fed-cost-acc}, we additionally visualize performance vs.\ communication costs and performance vs.\ epochs in Figure~\ref{fig:append-fed-tasks}. Although our method causes performance fluctuation in some datasets, the performance recovers very quickly. Figure~\ref{fig:append-fed-accer-bar} presents Figure~\ref{fig:fed-accer} in bar charts. The results show that our method saves considerable costs in almost all settings except for EMNIST at 95\%. It is because both baseline and our method improve fast at the beginning while our method stands out in the latter phase of training (e.g. after 98\%). The result also supports that our method improves across datasets.

We present Figure~\ref{fig:fed-accer} in a region that models are applicable. Here, we plot the figure with a larger range $\{50\%$, $60\%$, $70\%$, $80\%$, $90\%$ $98\%$, $99\%$, $99.95\%$, $\textit{best}\}$ in Figure~\ref{fig:append-fed-accer}. The result is consistent, showing that our method benefits across datasets, and is efficient when granted limited training budgets.

\begin{figure}[tbp]
\setlength{\tabcolsep}{1pt}
\centering
\begin{tabular}{ccccc}
 \includegraphics[width=0.16\linewidth]{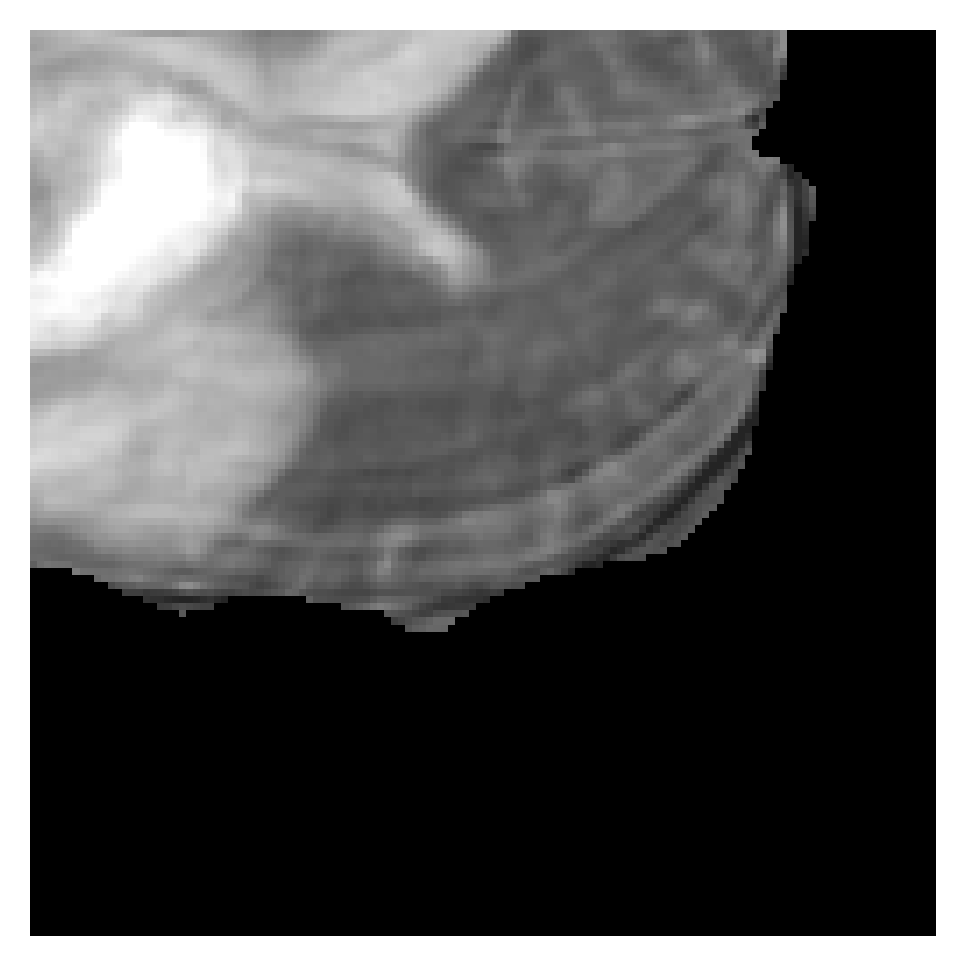} & \includegraphics[width=0.16\linewidth]{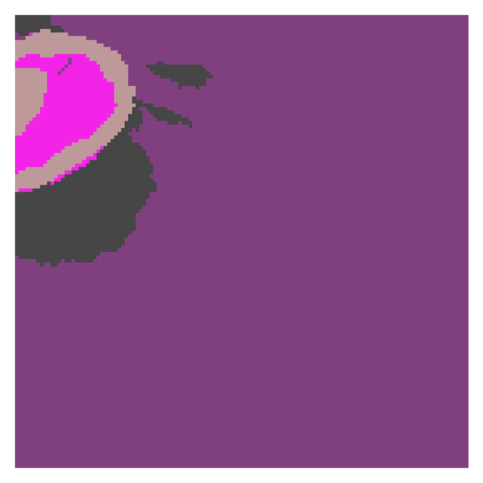} & \includegraphics[width=0.16\linewidth]{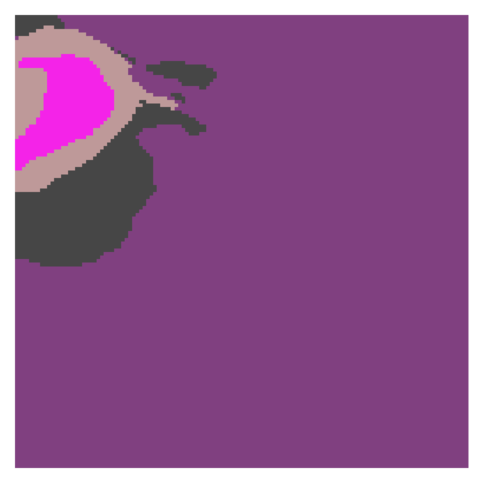} & \includegraphics[width=0.16\linewidth]{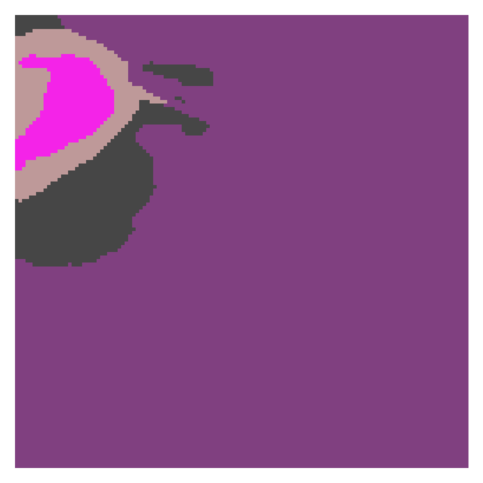} & \includegraphics[width=0.16\linewidth]{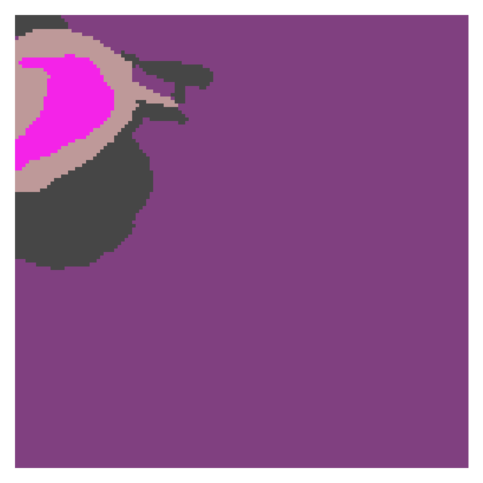} \\
 \includegraphics[width=0.16\linewidth]{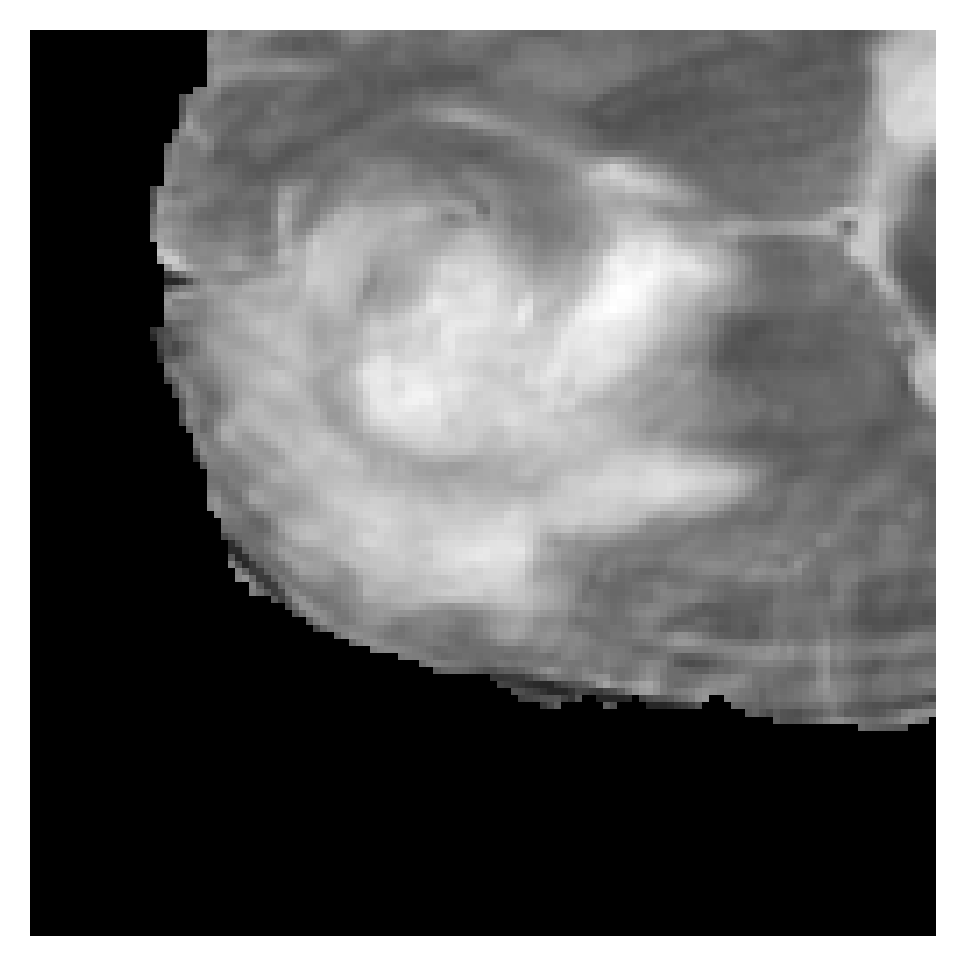} & \includegraphics[width=0.16\linewidth]{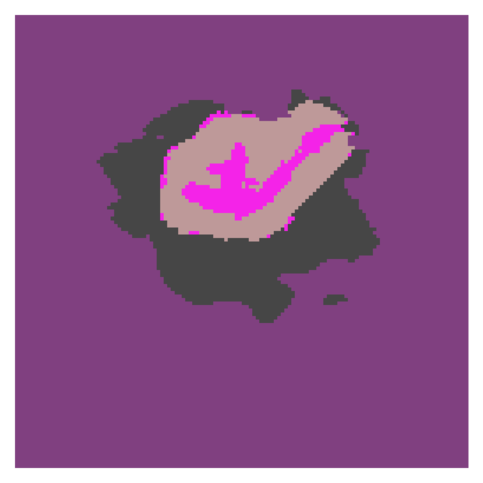} & \includegraphics[width=0.16\linewidth]{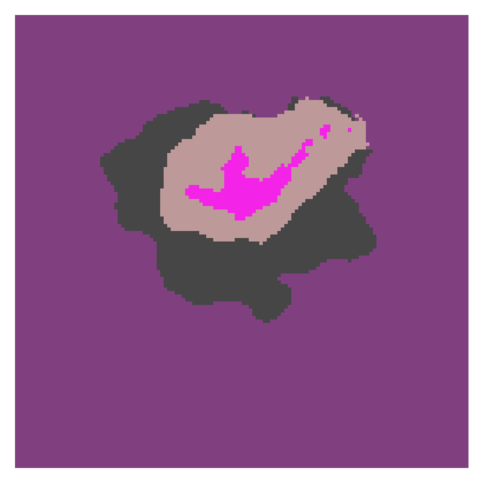} & \includegraphics[width=0.16\linewidth]{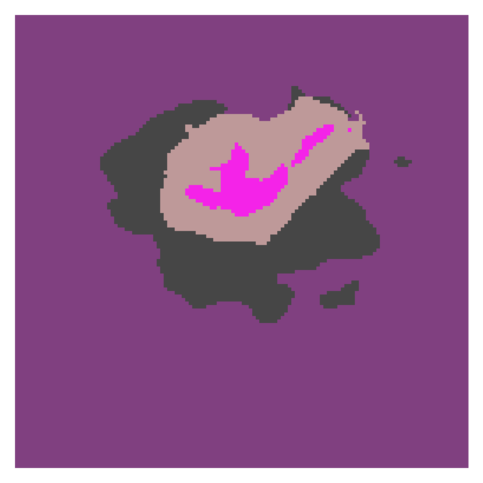} & \includegraphics[width=0.16\linewidth]{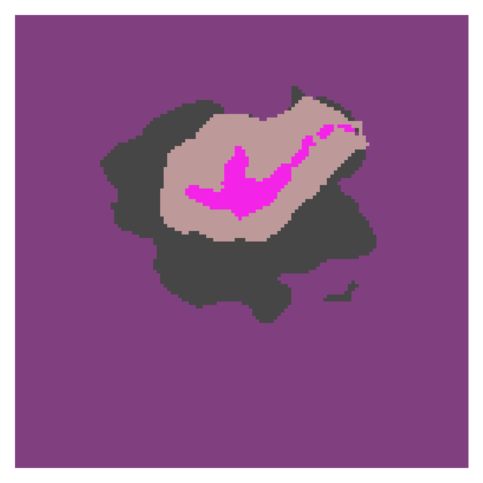} \\
 \includegraphics[width=0.16\linewidth]{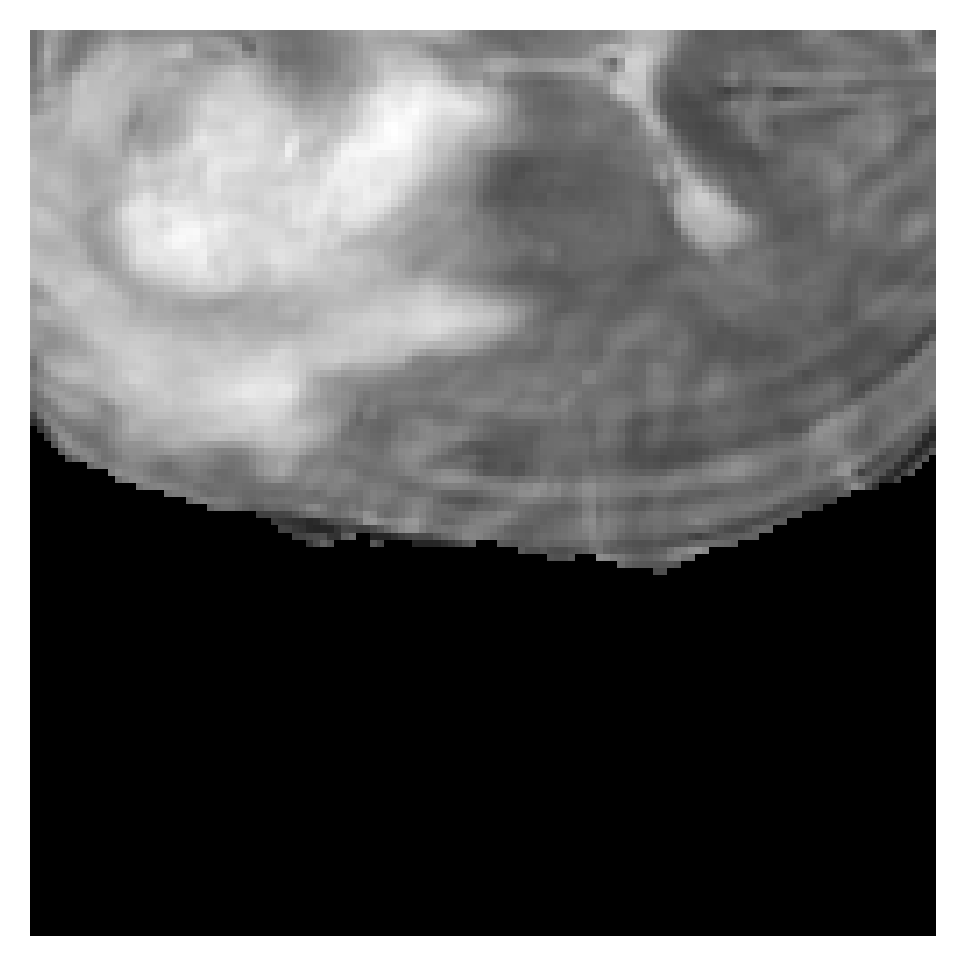} & \includegraphics[width=0.16\linewidth]{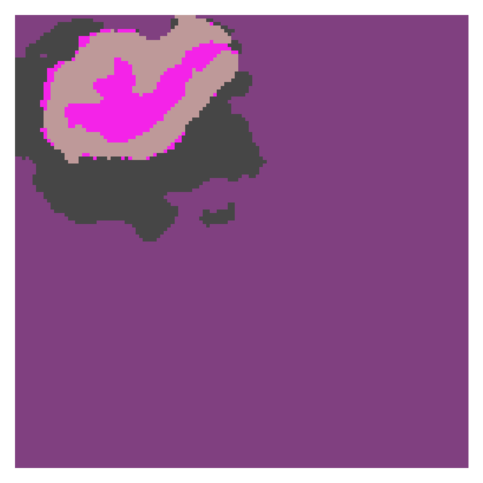} & \includegraphics[width=0.16\linewidth]{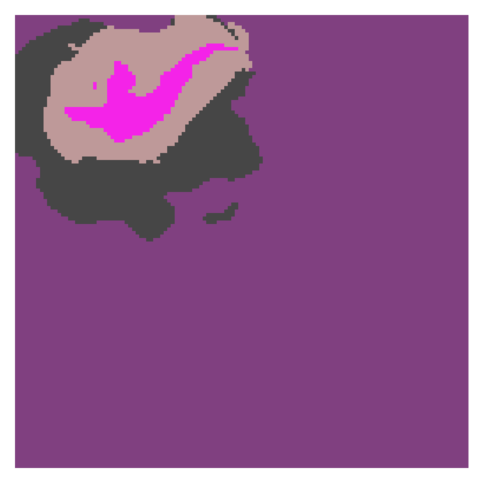} & \includegraphics[width=0.16\linewidth]{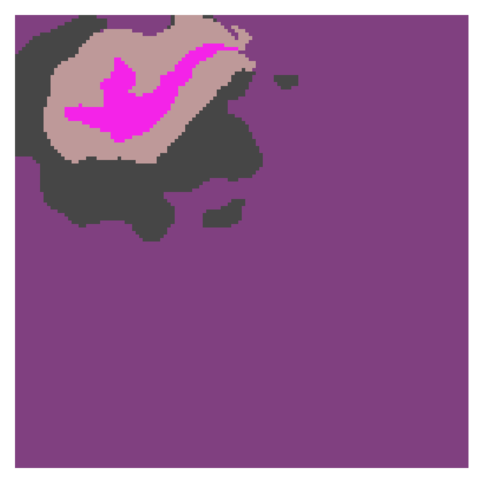} & \includegraphics[width=0.16\linewidth]{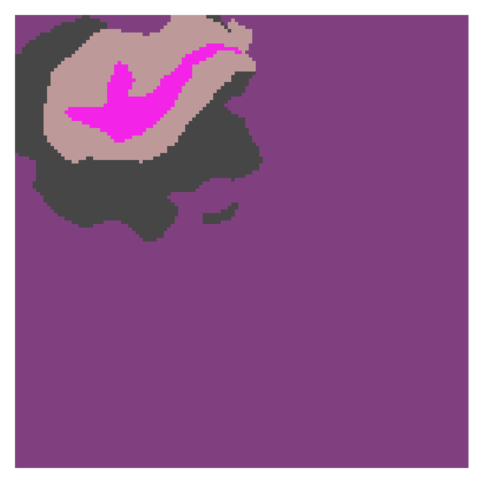} \\
{\small Input} & {\small Ground Truth} & {\small\begin{tabular}[c]{@{}c@{}}Baseline\\ (100\%)\end{tabular}} & {\small\begin{tabular}[c]{@{}c@{}}Asymmetric\\ (94.98\%)\end{tabular}} & {\small \begin{tabular}[c]{@{}c@{}}Symmetric\\ (36.40\%)\end{tabular}}
\end{tabular}
\caption{Visualization of federated segmentation. From left to right:\ Input, Ground Truth, Baseline, Asymmetric, and Symmetric updating strategies. Despite the comparable performance, \textit{Symmetric} consumes significantly fewer communication costs.}
\label{fig:append-vis-seg}
\end{figure}

\begin{figure}[bp]
\setlength{\tabcolsep}{1pt}
\centering
\begin{tabular}{ccccc}
  \includegraphics[width=0.16\linewidth]{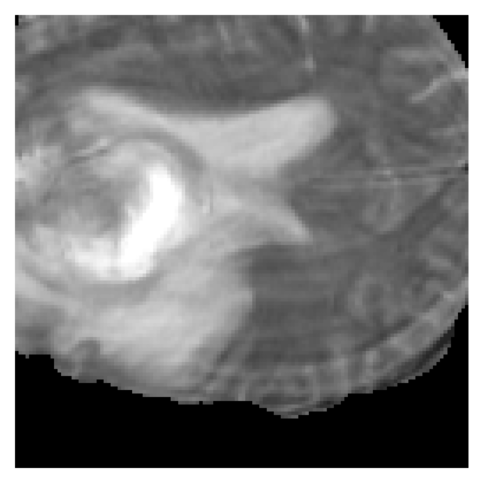} &   & \includegraphics[width=0.16\linewidth]{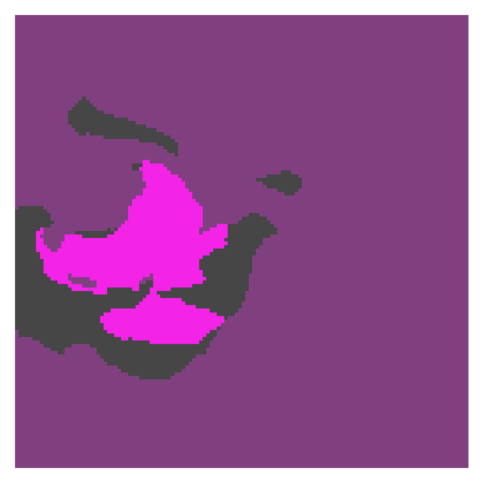} & \includegraphics[width=0.16\linewidth]{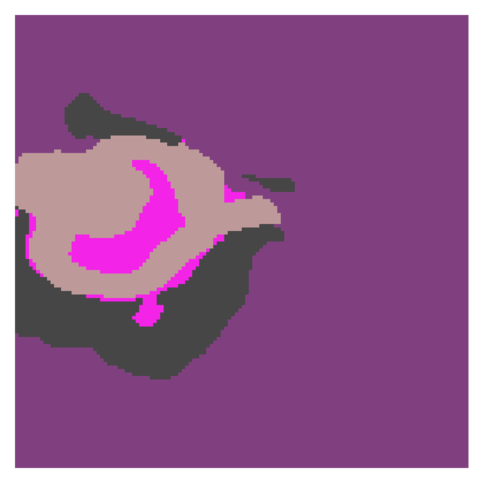} & \includegraphics[width=0.16\linewidth]{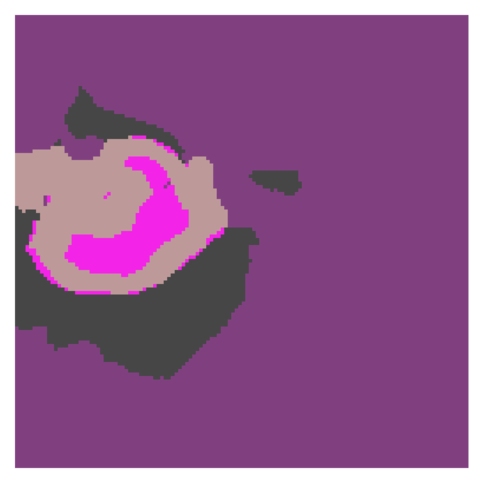} \\
 \includegraphics[width=0.16\linewidth]{images/vis-seg/x.png} &   & \includegraphics[width=0.16\linewidth]{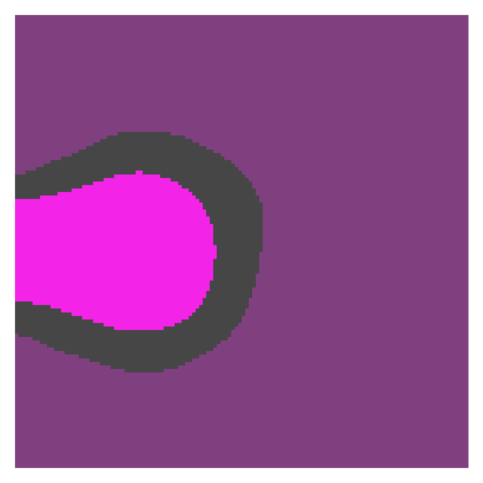} & \includegraphics[width=0.16\linewidth]{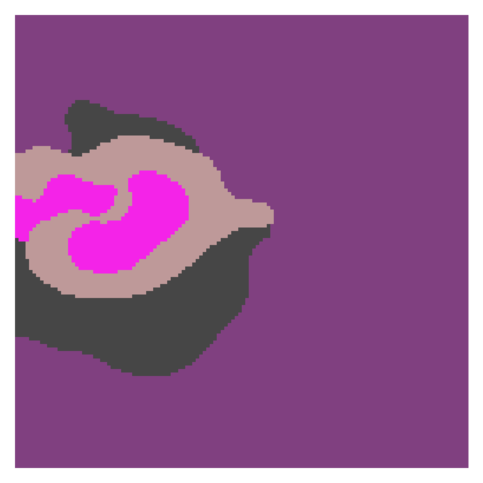} & \includegraphics[width=0.16\linewidth]{images/vis-seg/gt.png} \\
 \includegraphics[width=0.16\linewidth]{images/vis-seg/x.png} &\includegraphics[width=0.16\linewidth]{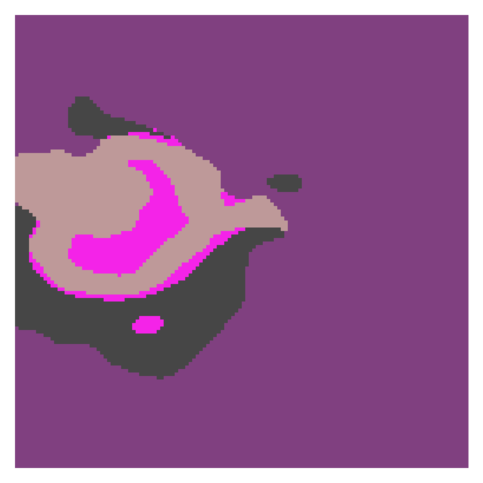}  & \includegraphics[width=0.16\linewidth]{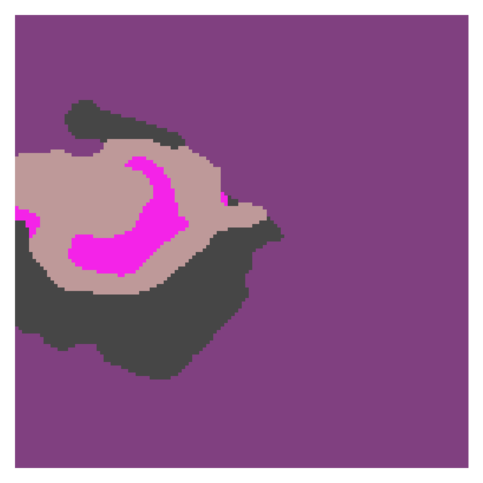} & \includegraphics[width=0.16\linewidth]{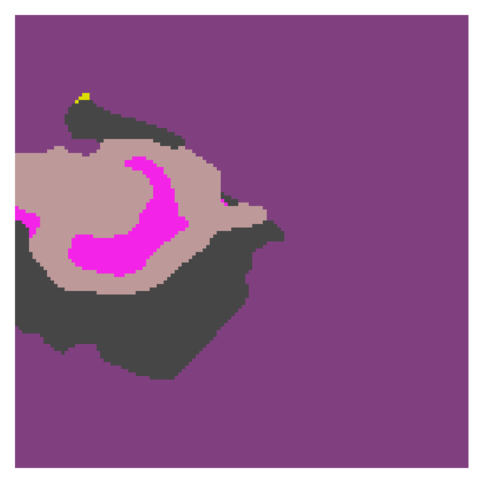} & \includegraphics[width=0.16\linewidth]{images/vis-seg/gt.png} \\
 {\small Input } &  {\small \begin{tabular}[c]{@{}c@{}} $\approx 0.18\%$  \end{tabular}} & {\small \begin{tabular}[c]{@{}c@{}} $\approx 9.40\%$  \end{tabular}} & {\small \begin{tabular}[c]{@{}c@{}} $\approx 36.40\%$  \end{tabular}} & {\small Ground Truth }
\end{tabular}
\caption{Segmentation results under \{$\approx0.18\%$, $\approx9.40\%$ , $\approx36.40\%$\} of communication costs of the converged baseline. From top to bottom: baseline, \textit{Asymmetric} (Ours), and \textit{Symmetric} (Ours). Only \textit{Symmetric} can achieve $0.18\%$ (6.7536 MB) compression ratio, since the size of the other models is already around 34 MB (i.e. $0.908\%$).}
\label{fig:append-seg-cost}
\end{figure}

\subsection{Visualization of Federated Segmentation}
\label{ssec:append-vis-seg}
We visualize the outputs of 3D-Unet (see Section~\ref{ssec:communication_efficiency}) in Figure~\ref{fig:append-vis-seg}. The result matches the number reported in Table~\ref{table:fed-tasks} that the models perform similarly after they have converged. However, the communication cost consumption is disparate. \textit{Symmetric} only consumes $36.40\%$ while the others either do not save any cost or marginally improve it. We visualize the results in Figure~\ref{fig:append-seg-cost} when granted limited communication budgets. Note that even with the same cost, \textit{Symmetric} and \textit{Asymmetric} may stay in different stages because \textit{Symmetric} starts from the outer part, consisting of significantly fewer parameters. We find that Baseline and \textit{Asymmetric} fail to compress the models with $0.18\%$ since their models are of size 34 MB (i.e. $0.908\%$), while only \textit{Symmetric} achieve it. Interestingly, \textit{Symmetric} has produced promising results at $0.18\%$ of costs. It suggests that our method significantly facilitates learning even given limited communication budgets. Meanwhile, we hope these findings could inspire more further work on medical learning problems.

\begin{table*}[tbp]
\centering
\caption{Raw results on CIFAR-100 with four architectures in the centralized setting (to complement Table~\ref{table:centralized_clf}).}
\label{table:append-centralized_clf}
\vskip 0.1in
\begin{tabular}{@{}llccccccccc@{}}
\toprule
 & \multicolumn{5}{c}{End-to-End} & \multicolumn{5}{c}{ProgFed (Ours)} \\ \cmidrule(l){2-11} 
 & Seed1 & Seed2 & Seed3 & Mean & \multicolumn{1}{c|}{Std} & Seed1 & Seed2 & Seed3 & Mean & Std \\
ResNet-18 & 75.95 & 76.12 & 76.17 & 76.08 & \multicolumn{1}{c|}{0.12} & 75.57 & 76.13 & 75.83 & 75.84 & 0.28 \\
ResNet-152 & 77.69 & 77.44 & 78.19 & 77.77 & \multicolumn{1}{c|}{0.38} & 78.95 & 78.46 & 78.31 & 78.57 & 0.33 \\
VGG16 (bn) & 71.94 & 71.77 & 71.65 & 71.79 & \multicolumn{1}{c|}{0.15} & 71.36 & 72.05 & 71.21 & 71.54 & 0.45 \\
VGG19 (bn) & 69.47 & 71.25 & 71.70 & 70.81 & \multicolumn{1}{c|}{1.18} & 71.30 & 70.45 & 70.95 & 70.90 & 0.43 \\ \bottomrule
\end{tabular}
\vskip -0.1in
\end{table*}

\begin{table}[tbp]
\centering
\caption{Raw results in federated settings on EMNIST, CIFAR-10, and CIFAR-100 (to complement Table~\ref{table:fed-tasks}).}
\label{table:append-fed-tasks}
\vskip 0.1in
\begin{tabular}{@{}lccccc@{}}
\toprule
 & \multicolumn{5}{c}{EMNIST} \\ \cmidrule(l){2-6} 
 & Seed1 & Seed2 & Seed3 & Mean & Std \\
Baseline & 85.63 & 85.77 & 85.85 & 85.75 & 0.11 \\
Ours & 85.65 & 85.62 & 85.73 & 85.67 & 0.06 \\ \midrule
 & \multicolumn{5}{c}{CIFAR-10} \\ \cmidrule(l){2-6} 
 & Seed1 & Seed2 & Seed3 & Mean & Std \\
Baseline & 84.62 & 84.82 & 84.56 & 84.67 & 0.14 \\
Ours & 84.64 & 84.73 & 85.19 & 84.85 & 0.30 \\ \midrule
 & \multicolumn{5}{c}{CIFAR-100} \\ \cmidrule(l){2-6} 
 & Seed1 & Seed2 & Seed3 & Mean & Std \\
Baseline & 51.79 & 51.86 & 52.58 & 52.08 & 0.44 \\
Ours & 53.33 & 53.16 & 53.21 & 53.23 & 0.09 \\ \bottomrule
\end{tabular}
\vskip -0.1in
\end{table}

\begin{table}[tbp]
\centering
\caption{Raw results on BraTS in the federated setting (to complement Table~\ref{table:fed-tasks}).}
\label{table:append-fed-seg}
\vskip 0.1in
\begin{tabular}{@{}lccccc@{}}
\toprule
 & Seed 1 & Seed 2 & Seed 3 & Mean & Std \\ \midrule
Baseline & 87.29 & 86.51 & 86.51 & 86.77 & 0.45 \\
Idea1 (Ours) & 88.19 & 87.22 & 87.57 & 87.66 & 0.49 \\
Idea2 (Ours) & 87.92 & 87.98 & 87.98 & 87.96 & 0.03 \\ \bottomrule
\end{tabular}
\vskip -0.1in
\end{table}

\begin{table*}[tbp]
\centering
\caption{Federated ResNet-18 on CIFAR-100 with compression. LQ-X denotes linear quantization followed by used bits representing gradients, and SP-X denotes sparsification followed by the percentage of kept gradients. (to complement Table~\ref{table:compress}).}
\label{table:append-compress}
\vskip 0.1in
\begin{tabular}{@{}lcccc@{}}
\toprule
 & Float & LQ-8 & LQ-4 & LQ-2 \\ \midrule
 & \multicolumn{4}{c}{Accuracy (\%)} \\ \cmidrule(l){2-5} 
Baseline & 52.08 $\pm$ 0.44 & 49.40 $\pm$ 0.75 & 49.55 $\pm$ 0.59 & 47.26 $\pm$ 0.29 \\
Ours & \textbf{53.23 $\pm$ 0.09} & \textbf{53.07 $\pm$ 1.00} & \textbf{52.32 $\pm$ 0.15} & \textbf{52.87 $\pm$ 0.54} \\ \midrule
 & \multicolumn{4}{c}{Compression ratio (\%)} \\ \cmidrule(l){2-5} 
Baseline & 100 & 25.00 & 12.50 & 6.25 \\
Ours & \textbf{77.10} & \textbf{19.28} & \textbf{9.64} & \textbf{4.82} \\ \midrule
 & SP-25 & SP-10 & \begin{tabular}[c]{@{}c@{}}LQ-8\\ +SP-25\end{tabular} & \begin{tabular}[c]{@{}c@{}}LQ-8\\ +SP-10\end{tabular} \\ \midrule
 & \multicolumn{4}{c}{Accuracy (\%)} \\ \cmidrule(l){2-5} 
Baseline & 51.23 $\pm$ 0.56 & 51.79 $\pm$ 0.10 & 49.67 $\pm$ 1.58 & 50.25 $\pm$ 1.03 \\
Ours & \textbf{52.00 $\pm$ 0.19} & \textbf{51.86 $\pm$ 0.23} & \textbf{52.19 $\pm$ 0.03} & \textbf{52.24 $\pm$ 0.12} \\ \midrule
 & \multicolumn{4}{c}{Compression ratio (\%)} \\ \cmidrule(l){2-5} 
Baseline & 25.00 & 10.00 & 6.25 & 2.50 \\
Ours & \textbf{19.28} & \textbf{7.71} & \textbf{4.82} & \textbf{1.93} \\ \bottomrule
\end{tabular}
\vskip -0.1in
\end{table*}

\subsection{Raw numbers and more statistics}
Table~\ref{table:append-centralized_clf}~\ref{table:append-fed-tasks}, and~\ref{table:append-fed-seg} present the raw numbers of the experiments in Table~\ref{table:centralized_clf} and \ref{table:fed-tasks} over three random seeds. Table~\ref{table:append-compress} presents the standard deviations of Table~\ref{table:compress} over three random seeds.

\begin{table*}[tbp]
\caption{Ablation study on numbers of stages with different numbers of total epochs on CIFAR-100.}
\label{table:append_ablation_s}
\centering
\vskip 0.1in
\begin{tabular}{@{}lcccc@{}}
\toprule
 & Accuracy (\%) & Cost (GB) & Accuracy (\%) & Cost (GB) \\ \midrule
 & \multicolumn{2}{c}{\#epoch=3000} & \multicolumn{2}{c}{\#epoch=4000} \\
End-to-End & 52.08 & 5368 & - & - \\
ProgFed (S=3) & 51.33 & 4418 & 53.80 & 5695 \\
ProgFed (S=4) & 53.23 & 4179 & - & - \\
ProgFed (S=5) & 51.53 & 4264 & 54.25 & 5479 \\
ProgFed (S=8) & 50.70 & 3901 & 54.46 & 5166 \\ \bottomrule
\end{tabular}
\vskip -0.1in
\end{table*}

\begin{table}[tbp]
\caption{Communication costs of ProgFed when achieving the performance of end-to-end training (52.08\%).}
\label{table:append_ablation_s_same_baseline}
\centering
\vskip 0.1in
\begin{tabular}{@{}lcc@{}}
\toprule
 & \#Total\_epochs & Cost \\ \midrule
End-to-End & 3000 & 5368 \\
ProgFed (S=3) & 4000 & 4365 \\
ProgFed (S=4) & 3000 & 3804 \\
ProgFed (S=5) & 4000 & 3781 \\
ProgFed (S=8) & 4000 & 3421 \\ \bottomrule
\end{tabular}
\vskip -0.1in
\end{table}

\begin{table}[tbp]
\caption{Performance of ProgFed when consuming the same amount of communication costs as S=4 and \#Total\_epochs = 3000 (4179GB).}
\label{table:append_ablation_s_same_cost}
\centering
\vskip 0.1in
\begin{tabular}{@{}lrr@{}}
\toprule
 & \multicolumn{1}{l}{\#Total\_epochs} & \multicolumn{1}{l}{Accuracy (\%)} \\ \midrule
End-to-end & 3000 & 51.19 \\
ProgFed (S=3) & 4000 & 52.33 \\
ProgFed (S=4) & 3000 & 53.23 \\
ProgFed (S=5) & 4000 & 53.56 \\
ProgFed (S=8) & 4000 & 53.50 \\ \bottomrule
\end{tabular}
\vskip -0.1in
\end{table}

\subsection{Ablation study on the number of stages}
\label{ssec:append_ablation_s}
We conduct an ablation study on CIFAR-100 with ResNet-18 to verify the influence of different numbers of stages $S$, namely $S=3$, $4$, $5$, and $8$. We set the remaining hyper-parameters as discussed in Section~\ref{ssec:proposed_method} to ensure a fair comparison. Table~\ref{table:append_ablation_s} shows that $S=4$ performs the best among all settings when the number of total epochs is equal to 3000. However, Theorem~\ref{thm:convergence} suggests that ProgFed could be at most two times slower than the end-to-end training. Therefore, we show that the performance immediately improves if we slightly increase the number of total epochs. Table~\ref{table:append_ablation_s} summarizes the result. Despite the improved performance, the final costs also increase subsequently. 

To fairly compare these two settings, we summarize the results in Table~\ref{table:append_ablation_s_same_baseline} when all settings achieve the end-to-end baseline performance (52.08\%). It is observed that all settings, including those with more epochs, achieve the same performance as the baseline at much fewer costs. Moreover, the cost decreases as the number of stages $S$ increases. On the other hand, we report the performance when all settings consume the same amount of costs as $S=4$ with 3000 epochs. Table~\ref{table:append_ablation_s_same_cost} shows that all settings perform better than the baseline and similarly well except $S=3$. It might indicate that ProgFed is insensitive to $S$ and the number of total epochs when given the same budget.

\section{More Related work}
\label{sec:append-related}
We discuss more related work in this section.

\myparagraph{Model pruning} Model pruning removes redundant weights to address the resource constraints~\citep{mozer1989skeletonization,lecun1990optimal,frankle2018the,lin2019dynamic}. There are two categories: \textit{unstructured} and \textit{structured} pruning. \textit{Unstructured} methods prune individual model weights according to certain criteria such as Hessian of the loss function~\citep{lecun1990optimal, hassibi1993second} and small magnitudes~\citep{han2015learning}. However, these methods cannot fully accelerate without dedicated hardware since they often result in sparse weights. In contrast, \textit{structured} pruning methods prune channels or even layers to alleviate the issue. They often learn importance weights for different components and only keep relatively important ones~\citep{liu2017learning, yu2018slimmable,  mohtashami2021simultaneous, li2021dynamic}. Despite the efficiency, model pruning usually happens at inference and does not reduce training costs while our method achieves computation- and communication-efficiency even during training. 

\myparagraph{Model distillation} Another line of research for model compression is model distillation~\citep{bucilua2006model,hinton2015distilling}, which requires a student model to mimic the behavior of a teacher model~\citep{polino2018model, sun2019patient}. Recent work has investigated transmitting logits rather than gradients~\citep{li2019fedmd, lin2020ensemble, he2020group, choquette2020capc}, which significantly reduces communication costs. However, these methods either require additional query datasets~\citep{li2019fedmd, lin2020ensemble} or cannot enjoy the merit of datasets from different sources~\citep{choquette2020capc}. In contrast, our work reduces the costs while retaining the dataset efficiency.

\end{document}